\documentclass{article}
\usepackage{ijcai22}

\usepackage{times}
\usepackage{soul}
\usepackage{url}
\usepackage[hidelinks]{hyperref}
\usepackage[utf8]{inputenc}
\usepackage[small]{caption}
\usepackage{graphicx}
\usepackage{amsmath}
\usepackage{amsthm}
\usepackage{booktabs}
\title{On the Complexity of Enumerating Prime Implicants \\from Decision-DNNF Circuits}

\author{
Alexis de Colnet$^1$
\And
Pierre Marquis$^{1,2}$
\affiliations
$^1$Univ. Artois, CNRS, Centre de Recherche en Informatique de Lens (CRIL), F-62300 Lens, France\\
$^2$Institut Universitaire de France
\emails
\{decolnet, marquis\}@cril.fr
}

\usepackage{todonotes}

\usepackage[ruled,vlined,linesnumbered]{algorithm2e}

\newcommand{\IP}{\textit{IP}}

\newcommand{\SR}{\textit{SR}}

\newcommand{\var}{var}

\newcommand{\calH}{\mathcal{H}}

\usetikzlibrary{hobby}
\usetikzlibrary{shapes}

\usepackage{amssymb}
\usepackage{xfrac}
\usepackage{mathtools}
\usepackage{enumitem}

\newtheorem{proposition}{Proposition}

\newtheorem*{corollary*}{Corollary}

\newtheorem{claim}{Claim}

\theoremstyle{definition}

\newtheorem{example}{Example}
\newtheorem{definition}{Definition}

\usepackage{tikz}
\usetikzlibrary{calc}
\usepackage{subcaption}[tiny]
\usetikzlibrary{patterns}

\begin{document}

\maketitle

\begin{abstract}
We consider the problem Enum$\cdot\IP$ of enumerating prime implicants of Boolean functions represented by decision decomposable negation normal form (dec-DNNF) circuits. We study Enum$\cdot\IP$ from dec-DNNF within the framework of enumeration complexity and prove that it is in {\sf OutputP}, the class of output polynomial enumeration problems, and more precisely in {\sf IncP}, the class of polynomial incremental time enumeration problems. We then focus on two closely related, but seemingly harder, enumeration problems where further restrictions are put on the prime implicants to be generated. In the first problem, one is only interested in prime implicants representing subset-minimal abductive explanations, a notion much investigated in AI for more than three decades.  In the second problem, the target is prime implicants representing sufficient reasons, a recent yet important notion in the emerging field of eXplainable AI, since they aim to explain predictions achieved by machine learning classifiers. We provide evidence showing that enumerating specific prime implicants corresponding to subset-minimal abductive explanations or to sufficient reasons is not in {\sf OutputP}.
\end{abstract}

\section{Introduction}

Prime implicants are a key concept when dealing with Boolean functions since the notion has been introduced seven decades ago \cite{Quine52}. 
%
Within AI, prime implicants (or the dual concept of prime implicates) have been considered for modeling and solving a number of problems,  
including compiling knowledge \cite{ReiterdeKleer87} and generating explanations of various kinds. 
This is the case in logic-based abductive reasoning (see e.g., \cite{SelmanLevesque90,EiterGottlob95}), a form of inference required in a number of applications when the available knowledge base is incomplete (e.g., in medicine) and because of such an incompleteness, it cannot alone explain the observations that are made about the state of the world. Abduction gave rise to much research in AI for the past three decades, especially because 
it is closely connected to other reasoning settings, including truth maintenance \cite{DeKleer86}, assumption-based reasoning and closed-world reasoning (see e.g., \cite{Marquis00} for a survey).  Formally, the explanations one looks for are terms over a preset  alphabet (composed of the so-called abducible variables, e.g., representing diseases) such that the manifestations that are reported (e.g., some symptoms) are logical consequences of the background knowledge when completed by such a term. In order to avoid trivial explanations, one also asks those terms to be consistent with the knowledge base. Explanations that are the less demanding ones from a logical standpoint (i.e., subset-minimal ones) can be characterized as specific prime implicants.
More recently, deriving explanations justifying why certain predictions have been made 
has appeared as essential for ensuring trustworthy Machine Learning (ML) technologies \cite{Miller19,Molnar19}. In the research area of eXplainable AI (XAI), recent work has shown how ML classifiers of various types (including black boxes) can be associated with Boolean circuits (alias transparent or ``white'' boxes), exhibiting the same input-output behaviours  \cite{DBLP:conf/aaai/NarodytskaKRSW18,ShihChoiDarwiche18b,ShihChoiDarwiche19}. 
Thanks to such mappings, XAI queries about classifiers can be delegated to the corresponding circuits. 
The notion of sufficient reason of an instance given a Boolean function $f$ modeling a binary classifier has been introduced in \cite{DarwicheH20}.  Given an instance $a$ ($a$ simply is an assignment, i.e., a vector of truth values given to each of the $n$ variables) such that $f(a) = 1$ (resp. $f(a) = 0$), a sufficient reason for $a$ is a subset-minimal partial assignment $a'$ which is coherent with $a$ (i.e., $a$  and $a'$ give the same values to the variables that are assigned in $a'$) and which satisfies the property that for every extension $a''$ of $a'$ we have $f(a'') = 1$ (resp. $f(a'') = 0$). The features assigned in $a'$ (and the way they are assigned) can be viewed as explaining why $a$ has been classified by $f$ as a positive (or as a negative) instance. 

Whatever the way prime implicants are used, generating them is in general a computationally demanding task, for at least two reasons. On the one hand, deriving a single prime implicant of a Boolean function represented by a propositional formula (or circuit) is {\sf NP}-hard since such a formula is satisfiable when it has a prime implicant, and it is valid precisely when this prime implicant is the empty term.  On the other hand, a source of complexity is the number of prime implicants that may prevent from computing them all. Indeed, it is well-known that the number of prime implicants of a Boolean function can be exponential in the number of variables of the function, and, for many representations of the function, also exponential in the size of the representation (just consider the parity function as a matter of example). In more detail, the number of prime implicants of a Boolean function can be  larger than the number of assignments satisfying the function \cite{DunhamFridshal59}; there also exist families of Boolean functions over $n$ variables having 
$\Omega(\frac{3^n}{n})$ prime implicants \cite{ChandraMarkowsky78}.

In this paper, we focus on the issue of \emph{enumerating prime implicants of a Boolean function represented by a decision decomposable negation normal form circuit (alias a dec-DNNF circuit)}. The question is to determine whether such prime implicants \emph{can be enumerated ``efficiently''}, which is obviously not the case when the circuit considered is unconstrained (as explained above, in such a case, computing a single prime implicant is already hard). This question is important for all the problems listed previously, when prime implicants represent explanations: since they are typically too numerous to be computed as a whole, it makes sense to derive them in an incremental way, with some performance guarantees in the generation; this lets the user who asked for an explanation deciding what to do after each derivation, namely to stop the enumeration process since he/she is satisfied by the explanation that has been provided, or alternatively to ask for a further explanation. 

The dec-DNNF language \cite{OztokDarwiche14,Darwiche01} and its subsets FBDD (free binary decision diagrams) \cite{GergovMeinel94}, OBDD (ordered binary decision diagrams) \cite{Bryant86} and even
DT (the set of all binary decision trees over Boolean variables, see e.g., ~\cite[Chapter~2]{Wegener00}) appear at first sight as good candidates for representing the function in the perspective of enumerating ``efficiently'' its prime implicants. Indeed, they are known as tractable representation languages (they support in polynomial time many queries and transformations from the so-called knowledge compilation map  \cite{DarwicheM02,Koricheetal13}).


\medskip
The main contribution of the paper is as follows. We give a polynomial incremental time algorithm for enumerating the prime implicants of a Boolean function $f$ represented by a dec-DNNF circuit $\Sigma$. Given $\Sigma$ and a positive integer $k$, this algorithm returns $k$ prime implicants of $\Sigma$ in $O(poly(k+|\Sigma|))$ time, or returns all prime implicants of $\Sigma$ if there are fewer than $k$. This shows that enumerating prime implicants from dec-DNNF is in the enumeration complexity class {\sf IncP}~\cite{Strozecki19}. We also provide evidence showing that enumerating specific prime implicants corresponding to subset-minimal abductive explanations or to sufficient reasons is not in {\sf OutputP}: on the one hand, computing a single subset-minimal abductive explanation from an OBDD circuit or a decision tree is {\sf NP}-hard; on the other hand, the existence of an output polynomial time algorithm for enumerating sufficient reasons from an OBDD circuit or a decision tree would lead to an output polynomial time algorithm for enumerating the minimal transversals of a hypergraph, thus answering a long-standing question related to monotone dualization~\cite{EiterMG08}.

The rest of the paper is organized as follows. We start with some preliminaries (Section \ref{sec:prelim}) where the language of dec-DNNF circuits and the framework used to study enumeration problems are presented. We formally define the problem Enum$\cdot\IP$ of enumerating prime implicants. Then in Section \ref{sec:outputp} we show that generating the set of \emph{all} prime implicants from a dec-DNNF circuit is feasible in output polynomial time. From there, we show in Section~\ref{sec:incp} that Enum$\cdot\IP$ from dec-DNNF is in fact in {\sf IncP} and point out a polynomial incremental time enumeration algorithm. Finally, in Section~\ref{sec:specific} we focus on subset-minimal abductive explanations and sufficient reasons and show that for each of the two cases, the enumeration issue is seemingly harder than in the case when all prime implicants are considered. 
All the proofs are reported in a final appendix.

\section{Preliminaries}\label{sec:prelim}

A Boolean function over $n$ variables $x_1,\dots,x_n$ is a mapping $f$ from $\{0,1\}^n$ to $\{0,1\}$. The set of variables of $f$ is denoted by $\var(f)$. The assignments to $\var(f)$ mapped to $1$ by $f$ are called \emph{satisfying assignments} of $f$. A literal upon variable $x$ is either $x$ or its negation $\overline{x}$ and a \emph{term} is a conjunction of literals. We often omit the conjunction symbols when writing terms, for instance we may shorten $a \land \overline{c}$ into $a\,\overline{c}$. We define the \emph{empty} term $t_\emptyset$ as the term over zero literal. The empty term verifies $t \land t_\emptyset = t$ for every term $t$. Given $\ell \in \{x,\overline{x}\}$, we denote by $f|\ell$ the Boolean function over $\var(f) \setminus \{x\}$ whose satisfying assignments coincide with that of $f \land \ell$. We use the usual symbols $\land$, $\lor$, $\neg$, $\models$ to denote conjunction, disjunction, negation, and entailment. Given a set $S$ of terms, $\max(S,\models)$ denotes the subset of terms of $S$ that do not entail another term in $S$. An \emph{implicant} of a Boolean function $f$ is a term $t$ whose satisfying assignments also satisfy $f$, i.e., $t \models f$. An implicant $t$ is \emph{prime} when no term $t - \ell$ obtained by removing a literal $\ell$ from $t$ is an implicant of $f$.

\subsection{Compilation Languages}

Compilation languages are often seen as classes of circuits. Let $PS$ be a countable set of propositional variables. A circuit in \emph{negation normal form} (NNF) is a directed acyclic graph (DAG) whose leaves are labelled with $0$ (\emph{false}), $1$ (\emph{true}), or a literal built upon $x \in PS$, and whose internal nodes are labelled with $\land$ or $\lor$ connectives; we call them $\land$-nodes and $\lor$-nodes. An NNF circuit computes a Boolean function over the variables appearing in it. For $v$ a node of an NNF circuit $\Sigma$, $\var(v)$ denotes the set of variables labelling leaves under $v$ in $\Sigma$ and $\Sigma_v$ denotes the subcircuit of $\Sigma$ rooted at $v$. The language of \emph{decomposable} NNF (DNNF) contains the NNF circuits where $\land$-nodes are decomposable, that is, the children $v_1,\dots, v_m$ of every $\land$-node $v$ are such that $\var(v_i) \cap \var(v_j) = \emptyset$ for all $i \neq j$. The language of \emph{deterministic, decomposable} NNF (d-DNNF) contains the DNNF circuits $\Sigma$ where $\lor$-nodes are deterministic, that is, the children $v_1,\dots, v_m$ of every $\lor$-node $v$ are such that $\Sigma_{v_i} \wedge \Sigma_{v_j}$ is inconsistent for all $i \neq j$. Finally, the language of \emph{decision} DNNF (dec-DNNF) is that of circuits whose leaves are labelled with $0$, $1$, or a literal built upon $x \in PS$, and whose internal nodes are decision nodes and $\land$-nodes. Whenever $n$ is a decision node labelled by variable $x$ in a dec-DNNF 
circuit $\Sigma$, the circuit $\Sigma_n$ given by
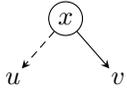
\raisebox{-0.5\height}{\begin{tikzpicture}
\node[circle, draw, inner sep=2] (x) at (0,0) {$x$};
\node[inner sep = 1.5] (u) at (-0.7,-0.8) {$u$};
\node[inner sep = 1.5] (v) at (+0.7,-0.8) {$v$};
\draw[-stealth, densely dashed] (x) -- (u);
\draw[-stealth] (x) -- (v);
\end{tikzpicture}} is viewed as a compact representation of the d-DNNF circuit
\raisebox{-0.5\height}{\begin{tikzpicture}
\node[inner sep = 1.5] (or) at (0,0) {$\lor$};
\node[inner sep = 1.5] (a1) at (-1+1/3,-0.3) {$\land$};
\node[inner sep = 1.5] (a2) at (2/3,-0.3) {$\land$};
\node[inner sep = 1.5] (nx) at (-1,-0.8) {$\overline{x}$};
\node[inner sep = 1.5] (x) at (1/3,-0.8) {$x$};
\node[inner sep = 1.5] (u) at (-1+2/3,-0.8) {$u$};
\node[inner sep = 1.5] (v) at (+1,-0.8) {$v$};
\draw (a2) -- (or) -- (a1);
\draw (u) -- (a1) -- (nx);
\draw (v) -- (a2) -- (x);
\end{tikzpicture}} (see Definition 2.6 and Figure 2 in~\cite{DarwicheM02}).

Thus, a decision node $n$ is labelled by a variable and has two children: the 0-child (node $u$ on the previous picture) and the 1-child (node $v$ on the previous picture). If $n$ is labelled by $x$ and $\Sigma_u$ (resp. $\Sigma_v$) represents the function $f_0$ (resp. $f_1$), then $\Sigma_n$ represents the function $(\overline{x} \land f_0) \lor (x \land f_1)$. For instance, Figure~\ref{fig:dec-DNNF} gives a dec-DNNF circuit whose deepest decision node computes $(\overline{s} \land 1) \lor (s \land p)$.
%
It is worth mentioning that all Boolean functions on finitely many variables can be represented in dec-DNNF, or indeed in any of its subsets like FBDD, OBDD, and DT.

Let $\Sigma$ be a dec-DNNF circuit.
The size of $\Sigma$, denoted by $|\Sigma|$ is its number of edges. From a dec-DNNF circuit $\Sigma$, one can easily derive in polynomial time a dec-DNNF circuit equivalent to $\Sigma$ where every $\land$-node has exactly two children. Since it is computationally harmless, for the sake of simplicity, our enumeration algorithms suppose that the dec-DNNF circuits satisfy this condition, so that their size is at most twice their number of nodes. In the same vein, we suppose that our dec-DNNF circuits have been \emph{reduced}, i.e., every node $v$ in $\Sigma$ such that $\Sigma_v$ computes the 0 function reduces to a leaf labelled by $0$. Testing the satisfiability of a dec-DNNF circuit is feasible in linear time \cite{DarwicheM02}, so reducing a dec-DNNF circuit also is a polynomial-time operation.

\subsection{Enumeration Complexity}

We now recall some enumeration complexity classes as described in~\cite{Strozecki19}. Let $V$ be an alphabet and let $A$ be a binary predicate in $V^* \times V^*$. Given an instance $x \in V^*$ (the input), $A(x)$ (the set of solutions) denotes the set of all $y \in V^*$ such that $A(x,y)$. The enumeration problem Enum$\cdot A$ is the function mapping $x$ to $A(x)$. Enum$\cdot A$ is in the class {\sf EnumP} if for every $y \in A(x)$, $|y|$ is polynomial in $|x|$, and if deciding whether $y$ is in $A(x)$ is in {\sf P}. {\sf EnumP} does not capture the complexity of \emph{computing} the set of solutions $A(x)$, it serves more as a counterpart of {\sf NP} for enumeration problems. 

The model used for the enumeration of solutions is the random access machine (RAM) model. See~\cite{Strozecki19} for details on why RAM have been chosen for this task. A RAM solves Enum$\cdot A$ if, for all $x$, it returns a sequence $y_1, \dots, y_m$ of pairwise distinct elements such that $\{y_1,\dots,y_m\} = A(x)$. Enum$\cdot A$ is in {\sf OutputP} if there is a RAM solving Enum$\cdot A$ in time $O(poly(|x|+|A(x)|))$ on every input~$x$. {\sf OutputP} is a relevant enumeration class when the whole set of solutions is explicitly asked for. For instance, the dualization of a monotone CNF formula $\phi$ is the task of generating a DNF formula equivalent to $\phi$. Because of the monotony condition on $\phi$, the terms used in any smallest DNF formula equivalent to $\phi$ are precisely its prime implicants. Thus, the dualization problem boils down to enumerating \emph{all} the prime implicants of $\phi$.

For other applications, computing only a fixed number of solutions may be enough. A RAM solves Enum$\cdot A$ in incremental time $f(t)g(n)$ if on every $x$, it runs in time time $O(f(t)g(|x|))$ and returns a sequence $y_1,\dots,y_t$ of $t$ pairwise distinct elements of $A(x)$ when $t \leq |A(x)|$, and the whole set $A(x)$ when $t > |A(x)|$. Enum$\cdot A$ is in {\sf IncP} if there is a RAM that solves $A$ in incremental time $O(t^an^b)$ for some constants $a$ and $b$. {\sf IncP} has a characterization that uses the function problem AnotherSol$\cdot A$ which, given $x$ and $S \subseteq A(x)$, returns $y \in A(x) \setminus S$ when $S \neq A(x)$, and \textit{false} otherwise.

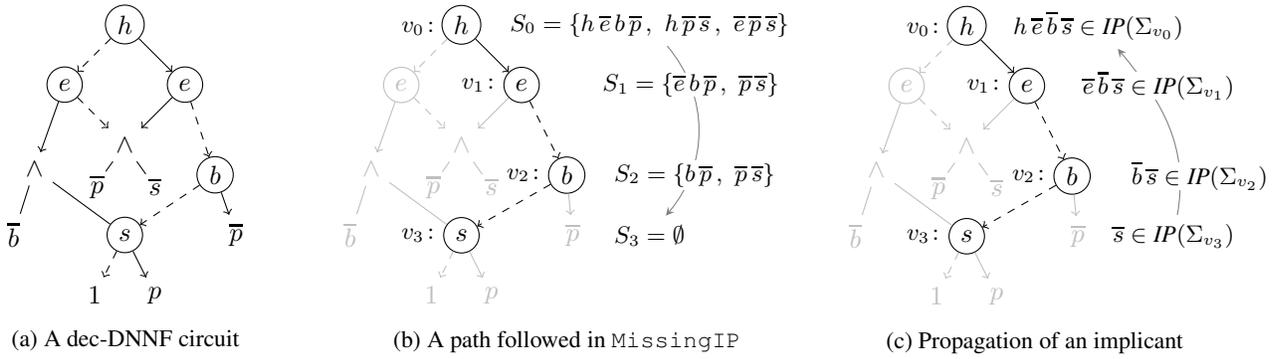
\begin{figure*}
\begin{subfigure}{0.25\textwidth}
\centering
\begin{tikzpicture}
\def\x{0.8};
\def\y{0.8};
\node[draw, circle, inner sep=2] (h) at (0,0) {$h$}; 
\node[draw, circle, inner sep=2.5] (h0) at (-\x,-\y) {$e$}; 
\node[draw, circle, inner sep=2.5] (h1) at (+\x,-\y) {$e$}; 

\draw[->, dashed] (h) -- (h0);
\draw[->] (h) -- (h1);

\node (h0e0) at (0*\x,-2*\y) {$\land$}; 
\node (h0e1) at (-1.5*\x,-2.4*\y) {$\land$}; 
\node[draw, circle, inner sep=2] (h1e0) at (1.5*\x,-2.5*\y) {$b$}; 

\draw[->, dashed] (h0) -- (h0e0);
\draw[->] (h0) -- (h0e1);
\draw[->] (h1) -- (h0e0);
\draw[->, dashed] (h1) -- (h1e0);

\node[inner sep=1.5] (h0e0l) at (-0.5*\x,-2.7*\y) {$\,\overline{p}$}; 
\node[inner sep=1.5] (h0e0r) at (+0.5*\x,-2.7*\y) {$\overline{s}$}; 

\draw (h0e0) -- (h0e0l);
\draw (h0e0) -- (h0e0r);

\node  (h0e1l) at (-1.85*\x,-3.5*\y) {$\overline{b}$}; 
\node[draw, circle, inner sep=2.5] (h0e1r) at (0*\x,-3.5*\y) {$s$}; 

\draw (h0e1) -- (h0e1l);
\draw (h0e1) -- (h0e1r);

\node (h1e0b1) at (1.85*\x,-3.5*\y) {$\overline{p}$};

\draw[->] (h1e0) -- (h1e0b1);
\draw[->, dashed] (h1e0) -- (h0e1r);

\node (h0e1rs0) at (-0.5*\x,-4.5*\y) {$1$};
\node (h0e1rs1) at (0.5*\x,-4.5*\y) {$p$};

\draw[->] (h0e1r) -- (h0e1rs1);
\draw[->, dashed] (h0e1r) -- (h0e1rs0);
%
%
%
%
%
%
%
\end{tikzpicture}
\caption{A dec-DNNF circuit}
\label{fig:dec-DNNF}
\end{subfigure}
\begin{subfigure}{0.4\textwidth}
\centering
\begin{tikzpicture}
\def\x{0.8};
\def\y{0.8};
\def\l{0.58};
\def\r{0.58};
\node[draw, circle, inner sep=2] (h) at (0,0) {$h$}; 
\node[font=\small] (l) at (0-\l,-0.05) {$v_0\!:$}; 
\node[draw, circle, inner sep=2.5, color=lightgray] (h0) at (-\x,-\y) {$e$}; 
\node[draw, circle, inner sep=2.5] (h1) at (+\x,-\y) {$e$}; 
\node[font=\small] (l) at (+\x-\l,-\y) {$v_1\!:$}; 

\draw[->, dashed, color=lightgray] (h) -- (h0);
\draw[->] (h) -- (h1);

\node[color=lightgray] (h0e0) at (0*\x,-2*\y) {$\land$}; 
\node[color=lightgray] (h0e1) at (-1.5*\x,-2.4*\y) {$\land$}; 
\node[draw, circle, inner sep=2] (h1e0) at (1.75*\x,-2.5*\y) {$b$};  
\node[font=\small] (l) at (+1.75*\x-\l,-2.5*\y) {$v_2\!:$}; 

\draw[->, dashed, color=lightgray] (h0) -- (h0e0);
\draw[->, color=lightgray] (h0) -- (h0e1);
\draw[->, color=lightgray] (h1) -- (h0e0);
\draw[->, dashed] (h1) -- (h1e0);

\node[inner sep=1.5, color=lightgray] (h0e0l) at (-0.5*\x,-2.7*\y) {$\,\overline{p}$}; 
\node[inner sep=1.5, color=lightgray] (h0e0r) at (+0.5*\x,-2.7*\y) {$\overline{s}$};   

\draw[color=lightgray] (h0e0) -- (h0e0l);
\draw[color=lightgray] (h0e0) -- (h0e0r);

\node[color=lightgray]  (h0e1l) at (-1.85*\x,-3.5*\y) {$\overline{b}$}; 
\node[draw, circle, inner sep=2.5] (h0e1r) at (0*\x,-3.5*\y) {$s$}; 
\node[font=\small] (l) at (0*\x-\l,-3.5*\y) {$v_3\!:$}; 

\draw[color=lightgray] (h0e1) -- (h0e1l);
\draw[color=lightgray] (h0e1) -- (h0e1r);

\node[color=lightgray] (h1e0b1) at (1.85*\x,-3.5*\y) {$\overline{p}$};

\draw[->, color=lightgray] (h1e0) -- (h1e0b1);
\draw[->, dashed] (h1e0) -- (h0e1r);

\node[color=lightgray] (h0e1rs0) at (-0.5*\x,-4.5*\y) {$1$};
\node[color=lightgray] (h0e1rs1) at (0.5*\x,-4.5*\y) {$p$};

\draw[->, color=lightgray] (h0e1r) -- (h0e1rs1);
\draw[->, dashed, color=lightgray] (h0e1r) -- (h0e1rs0);

\node[font=\small] (S3) at (0*\x+2.5,-3.5*\y) {$S_3 = \emptyset$}; 
\node[font=\small] (S0) at (0+2.5,0) {$S_0 = \{h\,\overline{e}\,b\,\overline{p} \,,\,\, h\,\overline{p}\,\overline{s} \,,\,\, \overline{e}\,\overline{p}\,\overline{s}\}$}; 
\draw[stealth-, color=gray] (S3) to [out=50, in=-50] (S0);
\node[font=\small,fill=white, inner sep=1] (S2) at (+1.75*\x+1.7,-2.5*\y) {$S_2 = \{b\,\overline{p} \,,\,\, \overline{p}\,\overline{s}\}$}; 
\node[font=\small,fill=white, inner sep=0]  (S1) at (+\x+2.25,-\y) {$S_1 = \{\overline{e}\,b\,\overline{p} \,,\,\, \overline{p}\,\overline{s} \}$}; 
\end{tikzpicture}
\caption{A path followed in \texttt{MissingIP}}
\label{fig:path}
\end{subfigure}
\begin{subfigure}{0.29\textwidth}
\centering
\begin{tikzpicture}
\def\x{0.8};
\def\y{0.8};
\def\l{0.58};
\def\r{0.58};
\node[draw, circle, inner sep=2] (h) at (0,0) {$h$}; 
\node[font=\small] (l) at (0-\l,-0.05) {$v_0\!:$}; 
\node[draw, circle, inner sep=2.5, color=lightgray] (h0) at (-\x,-\y) {$e$}; 
\node[draw, circle, inner sep=2.5] (h1) at (+\x,-\y) {$e$}; 
\node[font=\small] (l) at (+\x-\l,-\y) {$v_1\!:$}; 

\draw[->, dashed, color=lightgray] (h) -- (h0);
\draw[->] (h) -- (h1);

\node[color=lightgray] (h0e0) at (0*\x,-2*\y) {$\land$}; 
\node[color=lightgray] (h0e1) at (-1.5*\x,-2.4*\y) {$\land$}; 
\node[draw, circle, inner sep=2] (h1e0) at (1.75*\x,-2.5*\y) {$b$};  
\node[font=\small] (l) at (+1.75*\x-\l,-2.5*\y) {$v_2\!:$}; 

\draw[->, dashed, color=lightgray] (h0) -- (h0e0);
\draw[->, color=lightgray] (h0) -- (h0e1);
\draw[->, color=lightgray] (h1) -- (h0e0);
\draw[->, dashed] (h1) -- (h1e0);

\node[inner sep=1.5, color=lightgray] (h0e0l) at (-0.5*\x,-2.7*\y) {$\,\overline{p}$}; 
\node[inner sep=1.5, color=lightgray] (h0e0r) at (+0.5*\x,-2.7*\y) {$\overline{s}$};   

\draw[color=lightgray] (h0e0) -- (h0e0l);
\draw[color=lightgray] (h0e0) -- (h0e0r);

\node[color=lightgray]  (h0e1l) at (-1.85*\x,-3.5*\y) {$\overline{b}$}; 
\node[draw, circle, inner sep=2.5] (h0e1r) at (0*\x,-3.5*\y) {$s$}; 
\node[font=\small] (l) at (0*\x-\l,-3.5*\y) {$v_3\!:$}; 

\draw[color=lightgray] (h0e1) -- (h0e1l);
\draw[color=lightgray] (h0e1) -- (h0e1r);

\node[color=lightgray] (h1e0b1) at (1.85*\x,-3.5*\y) {$\overline{p}$};

\draw[->, color=lightgray] (h1e0) -- (h1e0b1);
\draw[->, dashed] (h1e0) -- (h0e1r);

\node[color=lightgray] (h0e1rs0) at (-0.5*\x,-4.5*\y) {$1$};
\node[color=lightgray] (h0e1rs1) at (0.5*\x,-4.5*\y) {$p$};

\draw[->, color=lightgray] (h0e1r) -- (h0e1rs1);
\draw[->, dashed, color=lightgray] (h0e1r) -- (h0e1rs0);

\node[font=\small] (i0) at (0+1.75,0) {$h\,\overline{e}\,\overline{b}\,\overline{s} \in \IP(\Sigma_{v_0})$};
\node[font=\small] (i3) at (0*\x+2.75,-3.5*\y) {$\overline{s} \in \IP(\Sigma_{v_3})$}; 
\draw[-stealth, color=gray] (i3) to [out=80, in=-50] (i0);
\node[font=\small,fill=white, inner sep=0] (i1) at (+\x+1.75,-\y) {$\overline{e}\,\overline{b}\,\overline{s} \in \IP(\Sigma_{v_1})$}; 
\node[font=\small,fill=white, inner sep=0] (i2) at (+1.75*\x+1.7,-2.5*\y) {$\overline{b}\,\overline{s} \in \IP(\Sigma_{v_2})$}; 
\end{tikzpicture}
\caption{Propagation of an implicant}
\label{fig:propagation}
\end{subfigure}
\caption{Generation of a new prime implicant from a dec-DNNF circuit}
\end{figure*}

\begin{proposition}[\cite{Strozecki19}]\label{prop:IncP_AnotherSol}
A problem Enum$\cdot A$ in {\sf EnumP} is in {\sf IncP} if and only if AnotherSol$\cdot A$ is in~{\sf FP}. 
\end{proposition}

Note that {\sf OutputP} is thought to be distinct from {\sf IncP}~\cite{Strozecki19}.

\section{Enum$\cdot$IP from dec-DNNF is in OutputP}\label{sec:outputp}

Let us first consider the problem of enumerating the prime implicants of a Boolean function $f$ given as a dec-DNNF circuit $\Sigma$, for short the prime implicants of $\Sigma$. Let $\IP(\Sigma,t)$ be the binary predicate representing the relation that $t$ is a prime implicant of $\Sigma$. Then $\IP(\Sigma)$ denotes the set of prime implicants of $\Sigma$. We extend the notation $\IP(\cdot)$ to any Boolean function~$f$. To be able to speak of prime implicants enumeration from circuits other than dec-DNNF ones we write ``Enum.$\IP$ from $L$'' with $L$ the language $\Sigma$ belongs to. 

We start with a couple of easy results. First of all, since there is a linear-time procedure to verify that a term is an implicant of a dec-DNNF circuit, there is a polynomial-time algorithm to decide whether a given a term is a prime implicant of a dec-DNNF circuits, thus:

\begin{proposition}
Enum$\cdot\IP$ from dec-DNNF is in {\sf EnumP}.
\end{proposition}

In addition, it is known that Enum$\cdot\IP$ from OBDD is in {\sf OutputP}~\cite{MadreC91}, and it is almost straightforward to extend this result to dec-DNNF. To make it precise, let us briefly describe the output polynomial construction of $\IP(\Sigma)$ for $\Sigma$, a dec-DNNF circuit. The construction is based on the three following, folklore propositions (for the sake of completeness, a proof for each of them is nonetheless reported as a supplementary material).

\begin{proposition}
\label{prop:decomposable_and_IP}
Let $f$ and $g$ be Boolean functions, then $\IP(f \land g) = \max(\{t \land t' \mid t \in \IP(f), \, t'\in \IP(g) \}, \models).$ Furthermore if $\var(f) \cap \var(g) = \emptyset$, then $\IP(f \land g) = \{t \land t' \mid t \in \IP(f), \, t'\in \IP(g) \}.$
\end{proposition}

\begin{proposition}
\label{prop:propagate_IP}
Let $f$ a Boolean function, let $x$ be a variable, and let $\ell \in \{x,\overline{x}\}$. Consider $t \in \IP(f|\ell)$. If $t \models f|\overline{\ell}$, then $t \in \IP(f)$, otherwise $t \land \ell \in \IP(f)$.
\end{proposition}

\begin{proposition}
\label{prop:decision_node_IP}
Let $f$ be a Boolean function and let $x$ be a variable.
\[
\begin{aligned}
\IP(f)
= \,     &\{t \land \overline{x} \mid t \in \IP(f|\overline{x}), \, t \not\models f|x \} \\
 \cup \, &\{t \land x \mid t \in \IP(f|x), \, t \not\models f|\overline{x} \} \\
 \cup \, &\IP(f|\overline{x} \land f|x)
\end{aligned}
\]
\end{proposition}

\noindent Note that $t \in \IP(f|\overline{x})$ (resp. $\IP(f|x)$) entails $f|x$ (resp. $f|\overline{x})$ if and only if $t$ is subsumed by some term in $\IP(f|\overline{x} \land f|x)$. As a consequence, from $\IP(f|\overline{x})$ and $\IP(f|x)$, one can construct $\IP(f|\overline{x} \land f|x)$ in polynomial time thanks to  Proposition~\ref{prop:decomposable_and_IP} and we use it to derive $\IP(f)$ thanks to Proposition~\ref{prop:decision_node_IP}.

We also have that (see the algorithm for conditioning a prime implicant representation provided in \cite{DarwicheM02}):

\begin{proposition}\label{prop:number_of_IP}
Let $f$ a Boolean function and let $x$ be a variable, then 
$|\IP(f)| \geq \max(|\IP(f|\overline{x})|, |\IP(f|x)|)$.
\end{proposition}

Consider now a dec-DNNF circuit $\Sigma$ and an internal node $v$ with two children $u$ and $w$. If the sets $\IP(\Sigma_u)$ and $\IP(\Sigma_w)$ are provided, then $\IP(\Sigma_v)$ is obtained in polynomial time using Proposition~\ref{prop:decomposable_and_IP} if $v$ is a decomposable $\land$-gate, and using Proposition~\ref{prop:decision_node_IP} if $v$ is a decision node. Furthermore, in both cases, we have $|\IP(\Sigma_v)| \geq \max(|\IP(\Sigma_u)|,|\IP(\Sigma_w)|)$. These observations lead to a simple algorithm that generates $\IP(\Sigma)$ by computing the sets $\IP(\Sigma_v)$ for every node $v$ of $\Sigma$ considered in a 
bottom-up way. Since constructing the set of prime implicants for any node given that of its children is tractable, since this set is smaller than $|\IP(\Sigma)|$, and since it is computed only once, the algorithm runs in time $O(poly(|\Sigma|+|\IP(\Sigma)|))$. Thus, we get:

\begin{proposition}\label{lemma:enum_ip_dec-DNNF_outputP}
Enum$\cdot\IP$ from dec-DNNF is in {\sf OutputP}.
\end{proposition}

\begin{example}
We give the construction of the sets of prime implicants for the nodes $v_1,v_2,v_3$ in the dec-DNNF circuit $\Sigma$ represented on Figure~\ref{fig:path}.
\begin{itemize}
\item[$v_3$:]  the sets of prime implicants of the children are $\IP(1) = \{t_\emptyset\}$ and $\IP(p) = \{p\}$. Using Proposition~\ref{prop:decision_node_IP} we have that $\overline{s}\land t_\emptyset = \overline{s}$ and $\Sigma_{v_3}|s = p$, so $\overline{s}\land t_\emptyset  \not\models  \Sigma_{v_3}|s$ showing that $\overline{s} \in \IP(\Sigma_{v_3})$. We also have that $\Sigma_{v_3}|\overline{s} = 1$, so $s\,p \models \Sigma_{v_3}|\overline{s}$ showing that $s\,p \not\in \IP(\Sigma_{v_3})$. Finally,   we have that $\IP(\Sigma_{v_3}|s \land \Sigma_{v_3}|\overline{s}) = \{p\}$ by Proposition~\ref{prop:decomposable_and_IP}, so $\IP(\Sigma_{v_3}) = \{\overline{s},p\}$.
\item[$v_2$:] the sets of prime implicants of the children are $\IP(\overline{p}) = \{\overline{p}\}$ and $\IP(\Sigma_{v_3})$ so we compute $\IP(\Sigma_{v_2}) = \{b\,\overline{p}, \overline{b}\,\overline{s}, \overline{b}\,p, \overline{p}\,\overline{s}\}$
\item[$v_1$:] the sets of prime implicants of the children are $\IP(\overline{p} \land \overline{s}) = \{\overline{p}\,\overline{s}\}$ and $\IP(\Sigma_{v_2})$ so we compute $\IP(\Sigma_{v_1}) = \{\overline{e}\,\overline{b}\,p, \overline{e}\,b\,\overline{p}, \overline{e}\,\overline{b}\,\overline{s}, \overline{p}\,\overline{s}\}$
\end{itemize}
\end{example}

\section{Enum$\cdot$IP from dec-DNNF is in IncP}\label{sec:incp}

We now investigate Enum$\cdot$IP from dec-DNNF from the incremental enumeration perspective. Based on Proposition~\ref{prop:IncP_AnotherSol}, we design a tractable algorithm \texttt{AnotherIP} for solving the problem AnotherSol$\cdot\IP$, thus showing that Enum$\cdot\IP$ from dec-DNNF is in {\sf IncP}.

\subsection{Solving the decision variant of AnotherSol$\cdot$IP}

We first consider the decision variant of AnotherSol$\cdot\IP$ from dec-DNNF: given a dec-DNNF circuit $\Sigma$ and a set $S \subseteq \IP(\Sigma)$, return \emph{false} if and only if $S \neq \IP(\Sigma)$. Recall from the discussion preceding Proposition~\ref{lemma:enum_ip_dec-DNNF_outputP} that there is a bottom-up procedure for generating all prime implicants of the dec-DNNF circuit $\Sigma$. To address the decision variant of AnotherSol$\cdot\IP$ on inputs $\Sigma$ and $S$, a reverse, top-down search is performed, assuming that $S$ is $\IP(\Sigma)$ until finding a contradiction. 

Before defining what a contradiction means in this setting, a few notations are useful. For $t$ a term and $X$ a set of variables, $t_X$ denotes the restriction of $t$ to variables in $X$. Note that if $X$ and $\var(t)$ are disjoint, then $t_X$ is the empty term~$t_\emptyset$.

\begin{proposition}\label{prop:decomposable_and_IP_check} 
Let $\Sigma$ be a dec-DNNF circuit and let $S \subseteq \IP(\Sigma)$. If the root of $\Sigma$ is an $\land$-node, let $u$ and $w$ be its children and let $S_u = \{t_{\var(\Sigma_u)} \mid t \in S\}$ and $S_w = \{t_{\var(\Sigma_w)} \mid t \in S\}$. Then $S_u \subseteq \IP(\Sigma_u)$ and $S_w \subseteq \IP(\Sigma_w)$ hold, and
\[\begin{aligned}
S = \IP(\Sigma) \text{ iff } & S_u = \IP(\Sigma_u) \text{ and } S_w = \IP(\Sigma_w) \\
\text{and } & S = \{t_u \land t_w \mid t_u \in S_u, t_w \in S_w\}.
\end{aligned}
\]
\end{proposition}

\begin{proposition}\label{prop:decision_node_IP_check}
Let $\Sigma$ be a dec-DNNF circuit whose root is a decision node labelled by $x$. Let $u$ be its 0-child and $w$ be its 1-child. Given $S \subseteq \IP(\Sigma)$, let $S_u = \{t \mid t \land \overline{x} \in S\}  \cup (S \cap \IP(\Sigma_{u}))$, $S_w = \{t \mid t \land x \in S\} \cup (S \cap \IP(\Sigma_{w}))$ and $S' = \{t \mid t \in S, x \not\in \var(t)\}$. Then $S_u \subseteq \IP(\Sigma_{u})$ and $S_w \subseteq \IP(\Sigma_{w})$ hold, and 
\[\begin{aligned}
S = \IP(\Sigma) \text{ iff } & S_{u} = \IP(\Sigma_{u})\text{ and } S_w = \IP(\Sigma_{w}) \\ \text{and } & S' = \max(\{t_u \land t_w \mid t_u \in S_u, \, t_w \in S_w \}, \models).
\end{aligned}
\]
\end{proposition}

Let $v$ be the root of the dec-DNNF circuit $\Sigma$ and let $S \subseteq \IP(\Sigma)$. We say that we have a \emph{contradiction} at node $v$ when 
\begin{itemize}
\item[(c1)] $S = \emptyset$ while $\Sigma$ is satisfiable, or
\item[(c2)] $v$ is a decision node, $S_u = \IP(\Sigma_u)$ and $S_w = \IP(\Sigma_{w})$, but $S' \neq \max(\{t_u \land t_w \mid t_u \in S_u, \, t_w \in S_{w} \}, \models)$, or
\item[(c3)] $v$ is a decomposable $\land$-node and $S_u =\IP(\Sigma_u)$ and $S_w = \IP(\Sigma_w)$ but $S \neq \{t_u \land t_w \mid t_u \in S_u, t_w \in S_w\}$.
\end{itemize}
A contradiction guarantees that $S \neq \IP(\Sigma)$. The contradiction (c1) is easy to check. Contradictions (c2) and (c3) on the other hand require to show that $S_u = \IP(\Sigma_u)$ and $S_w = \IP(\Sigma_w)$. When $v$ is an internal node, with children $u$ and $w$, if there is no contradiction (c1) at $v$, we use Propositions~\ref{prop:decomposable_and_IP_check} and~\ref{prop:decision_node_IP_check} to build from $S$ two sets $S_u$ and $S_w$ that we recursively compare to $\IP(\Sigma_u)$ and $\IP(\Sigma_w)$. Either the recursion ends under $u$ or $w$ on a contradiction, in which case $S \neq \IP(\Sigma)$, or it stops by itself (i.e., when reaching the leaves of the circuit), which shows that $S_u = \IP(\Sigma_u)$ and $S_w = \IP(\Sigma_w)$, and then we can check whether there is contradiction (c2) (resp. (c3)) at node $v$ if it is a decision node (resp. decomposable node). If there is none, then $S = \IP(\Sigma)$. 

The procedure is given by Algorithm \texttt{MissingIP}. The inputs are a dec-DNNF circuit $\Sigma$, a set $S \subseteq \IP(\Sigma)$ and a path $P$ in $\Sigma$ (which will be useful later). A function $\lambda$ mapping the nodes of $\Sigma$ to integers is used for memoization purposes. Initially $\lambda(v) = -1$ for every node $v$, but $\lambda(v)$ may be assigned a non-negative value at some point. More precisely, the first time a call \texttt{MissingIP}$(\Sigma_v, S, P)$ returns \textit{false}, we learn that $S = \IP(\Sigma_v)$ and set $\lambda(v)$ to $|S|$. Then for each later call \texttt{MissingIP}$(\Sigma_v,S',P')$ with $S' \subseteq \IP(\Sigma_v)$, we check whether $S' = \IP(\Sigma_v)$ by verifying that $\lambda(v) = |S'|$.

\begin{proposition}\label{prop:missingIP_is_sound} 
Given a reduced dec-DNNF circuit $\Sigma$ and $S \subseteq \IP(\Sigma)$, 
\texttt{\textup{MissingIP}}$(\Sigma,S,\emptyset)$ runs in time $O(poly(|S|+|\Sigma|))$, and it returns \textit{false} if and only if $S = \IP(\Sigma)$.
\end{proposition}

\let\oldnl\nl
\newcommand{\nonl}{\renewcommand{\nl}{\let\nl\oldnl}}
\begin{algorithm}[t]
\SetAlgoLined
\nonl\textbf{Promises}: $\Sigma$ is reduced, $S \subseteq \IP(\Sigma)$\\
Let $v$ be the root of $\Sigma$ and let $P' \leftarrow P \cup (v)$\\
\textbf{if} $\lambda(v) = |S|$ \textbf{then} return \textit{false}\\
\If{$S = \emptyset$}{
\textbf{if} $v$ is labelled by $0$ \textbf{then} set $\lambda(v)$ to 0, return \textit{false}\\
\textbf{else} return (\texttt{GenerateIP}($\Sigma$), $P'$)
}
\uIf{$v$ \text{ is a }$\land$\textup{-node with children} $u$\textup{ and }$w$}{
Build $S_u$ and $S_w$ as in Proposition~\ref{prop:decomposable_and_IP_check}\\
$r \leftarrow \texttt{MissingIP}(\Sigma_u,S_u,P')$\\ 
\textbf{if} $r \neq \textit{false}$ \textbf{then} return $r$\\
$r \leftarrow \texttt{MissingIP}(\Sigma_w,S_w,P')$\\ 
\textbf{if} $r \neq \textit{false}$ \textbf{then} return $r$\\
$S^* \leftarrow \{t_u \land t_v \mid t_u \in S_u, t_w \in S_w\}$
\\
\textbf{if} $S \neq S^*$ \textbf{then} for any $t \in S^* \setminus S$ return $(t,P')$\\
}
\ElseIf{$v$ \textup{ is a decision node with children $u$ and $w$}}{
Build $S_u$, $S_w$, $S'$ as in Proposition~\ref{prop:decision_node_IP_check}\\
$r \leftarrow \texttt{MissingIP}(\Sigma_u,S_u,P')$\\ 
\textbf{if} $r \neq \textit{false}$ \textbf{then} return $r$\\
$r \leftarrow \texttt{MissingIP}(\Sigma_w,S_w,P')$\\
\textbf{if} $r \neq \textit{false}$ \textbf{then} return $r$\\
$S^* \leftarrow \max(\{t_u \land t_w \mid t_u \in S_u, t_w \in S_w\}, \models)$\\
\textbf{if} $S^* \neq S'$ \textbf{then} for any $t \in S^* \setminus S'$ return $(t,P')$ 
}
Set $\lambda(v)$ to $|S|$ and return $\textit{false}$
\caption{\texttt{MissingIP}$(\Sigma,S,P)$}
\end{algorithm}

\subsection{Augmenting an incomplete subset of IP$(\Sigma)$}

We build upon \texttt{MissingIP} so that, when $S \neq \IP(\Sigma)$, we also return a prime implicant in $\IP(\Sigma)\setminus S$. The idea is to use the path $P$ to keep track of the ancestor nodes that were visited before reaching a contradiction and to use $P$ to construct a prime implicant in $\IP(\Sigma)\setminus S$.  As an example, consider calling \texttt{MissingIP}$(\Sigma,S_0,\emptyset)$ with $\Sigma$ the dec-DNNF circuit of Figure~\ref{fig:dec-DNNF} and $S_0 = \{h\,\overline{e}\,b\,\overline{p} ,\, h\,\overline{p}\,\overline{s},\,\overline{e}\,\overline{p}\,\overline{s}\}$ a set of prime implicants of $\Sigma$. Figure~\ref{fig:path} shows a scenario when  \texttt{MissingIP}$(\Sigma,S_0,\emptyset)$ calls \texttt{MissingIP}$(\Sigma_{v_1},S_1,(v_0))$, which calls in turn\texttt{MissingIP}$(\Sigma_{v_2},S_2,(v_0,v_1))$, which finally calls \texttt{MissingIP}$(\Sigma_{v_3},S_3,(v_0,v_1,v_2))$. Since $S_3 = \emptyset$ and $\Sigma_{v_3}$ is reduced and different from $0$, the algorithm has reached a contradiction  (c1) at node $v_3$ and has not returned \emph{false}, thus indicating that $S_0 \neq \IP(\Sigma)$. \texttt{MissingIP} has followed the path $P=(v_0,v_1,v_2,v_3)$ to reach that contradiction and has kept it in memory. This path $P$ can then be used to generate a prime implicant in $\IP(\Sigma) \setminus S_0$. First \texttt{MissingIP} returns the path $P$ to $v_3$ as well as a prime implicant of $\Sigma_{v_3}$, say it is $\overline{s}$. Then we construct a prime implicant of $\Sigma_{v_2}$ upon $\overline{s}$, here since $v_3$ is the 0-child of $v_2$ and since $\overline{s}$ does not entail the 1-child of $v_2$ we obtain $\overline{b}\,\overline{s} \in \IP(\Sigma_{v_2})$. Then we construct a prime implicant of $\Sigma_{v_1}$ upon $\overline{b}\,\overline{s}$, here since since $v_2$ is the 0-child of $v_1$ and since $\overline{b}\,\overline{s}$ does not entail the 1-child of $v_1$ we obtain $\overline{e}\,\overline{b}\,\overline{s} \in \IP(\Sigma_{v_1})$. Repeating the step one more time leads to $h\,\overline{e}\,\overline{b}\,\overline{s} \in \IP(\Sigma_{v_0}) = \IP(\Sigma)$. The procedure is illustrated in Figure~\ref{fig:propagation}. In this example, for generating a new prime implicant of $\Sigma$, we have created $t \in \Sigma_{v_3}\setminus S_3$ and augmented it using  Proposition~\ref{prop:propagate_IP} as we travelled backwards along $P$. We say that we have \emph{propagated} $t$ along the path $P$.

Accordingly, the algorithm \texttt{AnotherIP} to generate a new prime implicant breaks into two steps. First \texttt{MissingIP}$(\Sigma,S,P)$ searches for a contradiction. It returns $false$ if $S = \IP(\Sigma)$ or a pair $(t,P)$ with $P$ the path followed to reach a node $v$ where a contradiction has been found (like $v_3$ in the example), and $t$ a prime of $\Sigma_v$ that could not be derived from $S$. The procedure \texttt{GenerateIP} is used to generate $t$. \texttt{GenerateIP} runs in polynomial time thanks to linear-time implicant check on dec-DNNF circuits. Finally \texttt{PropagateIP} is called to propagate $t$ along the path $P$. 

\begin{algorithm}[t]
\SetAlgoLined
\nonl\textbf{Promise}: $\Sigma$ is satisfiable\\
Find a satisfying assignment $a$ of $\Sigma$
\\
Let $t = \bigwedge_{a(x) = 1} x \land \bigwedge_{a(x) = 0} \overline{x} $
\\
\While{\textup{there is} $\ell \in t$ \textup{such that} $t - \ell \models \Sigma$}{
Remove $\ell$ from $t$
}
Return $t$
\caption{\texttt{GenerateIP}$(\Sigma)$}
\end{algorithm} 

The next proposition shows the correctness of \texttt{\textup{AnotherIP}}:

\begin{algorithm}[t]
\SetAlgoLined
\nonl\textbf{Promise:} $\Sigma$ is reduced, its root is $v_0$, $P$ is a path in $\Sigma$\\
\textbf{if} $|P|=1$ \textbf{then} return $t$ \\
\uIf{$v_{i-1}$\textup{ is a} $\land$\textup{-node with children } $u$ \textup{ and } $w$}{
	\textbf{if} $v_i = u$ \textbf{then} $t' \leftarrow \texttt{GenerateIP}(\Sigma_w)$ \\
	\textbf{if} $v_i = w$ \textbf{then} $t' \leftarrow \texttt{GenerateIP}(\Sigma_u)$ \\
}\ElseIf{\textup{$v_{i-1}$ is a decision node for variable $x$ with 0-child $u$ and 1-child $w$}}{
	\uIf{$v_i = u$}{
		\textbf{if} $t \models \Sigma_{w}$ \textbf{then} $t' \leftarrow  t_{\emptyset}$ \textbf{else} $t' \leftarrow  \overline{x}$  
	}
	\uElse{
		\textbf{if} $t \models \Sigma_{u}$ \textbf{then} $t' \leftarrow  t_{\emptyset}$ \textbf{else} $t' \leftarrow  x$  
	}
}
$\texttt{Propagate}(\Sigma,t\land t',(v_0,\dots,v_{i-1}))$
\caption{\mbox{\texttt{Propagate}$(\Sigma,t,P = (v_0,\dots,v_i))$}}
\end{algorithm}
\begin{algorithm}[t]
\SetAlgoLined
\nonl\textbf{Promise}: $\Sigma$ is reduced, $S \subseteq \IP(\Sigma)$\\
$r \leftarrow \texttt{MissingIP}^*(\Sigma,S,\emptyset)$\\
\textbf{if} $r = false$ \textbf{then} return $false$ \\
\textbf{else if} $r = (t,P)$ \textbf{then} return $\texttt{Propagate}(\Sigma,t,P)$
\caption{\texttt{AnotherIP}$(\Sigma,S)$}
\end{algorithm}

\begin{proposition}\label{prop:AnotherIP}
Let $\Sigma$ be a reduced dec-DNNF circuit and let $S \subseteq \IP(\Sigma)$. \texttt{\textup{AnotherIP}}$(\Sigma, S)$ runs in time $O(poly(|S|+|\Sigma|))$. It returns $\textit{false}$ if $S = \IP(\Sigma)$, otherwise it returns a prime implicant of $\Sigma$ that does not belong to in $S$.
\end{proposition}

On this basis, the existence of a polynomial incremental time enumeration of prime implicants for dec-DNNF circuits
can be easily established:

\begin{proposition}\label{lemma:IP_enum_for_dec_DNNF_is_in_IncP}
Enum$\cdot\IP$ from dec-DNNF is in {\sf IncP}.
\end{proposition}  


\section{Enumerating Specific Prime Implicants }\label{sec:specific}

For some applications, enumerating all prime implicants of $f$ makes sense, even though there can be exponentially many. We have already mentioned the dualization of monotone CNF formulae as an example. In this section, we describe two problems that ask for generating only specific prime implicants, representing respectively subset-minimal abductive explanations and sufficient reasons. 

To illustrate the two notions we use the function $f$ computed by the dec-DNNF circuit of Figure~\ref{fig:dec-DNNF} as a toy example. $f$ encodes a very incomplete characterization of human-like creatures in Tolkien's Middle Earth based on four physical attributes: presence of beard and facial hair ($b$), small size ($s$), human-like skin ($h$), pointy ears ($p$), plus the indication of whether the creature is enrolled in the armies of evil ($e$). We imagine that there are only seven possible creatures: hobbits ($h\,\overline{b}\,p\,s\,\overline{e}$), elves ($h\,\overline{b}\,p\,\overline{s}\,\overline{e}$), dwarfs ($h\,b\,\overline{p}\,s\,\overline{e}$), men and women ($h\ast\overline{p}\,\overline{s}\,\ast$),\footnote{$\ast$ denotes that both choices are possible for the variable, typically here humans may fight for evil, humans and ents may or may not have beards, and orcs have a wide range of size.} ents ($\overline{h}\ast\overline{p}\,\overline{s}\,\overline{e}$), orcs ($\overline{h}\,\overline{b}\,p\ast e$) and trolls ($\overline{h}\,\overline{b}\,\overline{p}\,\overline{s}\,e$). The satisfying assignments of $f$ describe these creatures. Its prime implicants are the smallest combinations of attributes which guarantee the existence of a creature in our Middle Earth.

\subsection{Abductive Explanations}

Abductive explanations  (see e.g., \cite{SelmanLevesque90,EiterGottlob95}) can be defined as follows:

\begin{definition}[\textbf{Abductive explanation}] Given a Boolean function $f$ over variables $X$, a subset $H \subseteq X$, and a term $m$ on $X \setminus H$, an \emph{abductive explanation} is a term $t$ on $H$ such that $f \land t$ is satisfiable and $f \land t \models m$.
\end{definition}

\noindent The abduction problem asks whether an abductive explanation $t$ exists for the input $(f,H,m)$. 

\begin{example}
Consider our toy example. We look for combinations of physical attributes that guarantee that the creature is evil. This is an abduction problem with $H = \{h,b,p,s\}$ and $m = e$. For instance the term $\overline{h} \land p$ is an abductive explanation because there exist creatures with pointy ears and a skin that is not human-like, and all of them are evil (in this case only the orcs fit this description).
\end{example}

It is easy to see that an abductive explanation $t$ is in fact an implicant of $\neg f \lor m$ with the conditions that $f \land t$ is satisfiable and that $t$ is restricted to variables in $H$ (the abducibles). Furthermore, since abduction is not a truth-preserving form of inference, one is often interested in generating subset-minimal abductive explanations only (i.e., the logically weakest abductive explanations); they correspond to the \emph{prime} implicants of $\neg f \lor m$ such that $f \land t$ is satisfiable and $t$ is restricted to variables in $H$. 

Obviously enough, the abduction problem we focus on (the existence of an abductive explanation) is the same, would we consider subset-minimal abductive explanations or not. Indeed, deciding whether an abductive explanation exists is equivalent to deciding whether a subset-minimal abductive explanation exists.
Unfortunately, the condition that only variables in $H$ are allowed in abductive explanations is already too demanding from an enumeration perspective. 

\begin{proposition}\label{prop:abductive}
Unless ${\sf P} = {\sf NP}$, there is no polynomial-time algorithm which, given an OBDD circuit or a decision tree computing a function $f$ over $X$ and a set $Y \subseteq X$, decides whether $f$ has an implicant $t$ with $\var(t) \subseteq Y$. 
\end{proposition}

\subsection{Sufficient Reasons}

The notion of sufficient reason\footnote{This concept is also referred to as ``abductive explanations'' \cite{IgnatievNM19,DBLP:journals/corr/abs-2012-11067}; in the following, we stick to ``sufficient reason'' to avoid any confusion with the (distinct) concept of abductive explanations as discussed in the previous section.} \cite{DarwicheH20} (aka prime implicant explanation \cite{ShihCD18}) is defined as follows:

\begin{definition}[\textbf{Sufficient reason}]
Given a Boolean function $f$, let $a$ be any assignment to a superset of $\var(f)$. A \emph{sufficient reason} for $a$ is a prime implicant $t$ of $f$ (resp. $\neg f$) such that $a$ satisfies $t$, provided that $a$ satisfies $f$ (resp. $\neg f$). The set of all sufficient reasons for $a$ given $f$ is denoted by $\SR(f,a)$ (resp. $\SR(\neg f,a)$) when $a$ satisfies $f$ (resp. $\neg f$).
\end{definition}

\begin{example}
Consider again our toy example. There is no creature which is small, has human-like skin, pointy ears, no facial hair, and is evil. Finding the reasons of why such a creature cannot exist, means finding sufficient reasons for the assignment $a$ defined by $a(h) = a(p) = a(s) = a(e) = 1$ and $a(b) = 0$ given $\neg f$. In this case $h\,p\,e \in SR(\neg f, a)$ explains why such a creature cannot exist: there are no creatures that are evil and have both human-like skin and pointy ears, but there are such creatures that are non-evil (hobbits and elves), and there are evil creatures that have pointy ears (orcs) or human-like skin (men). There are other sufficient reasons for $a$ given $\neg f$, for instance $h\,s\,e  \in SR(\neg f, a)$. 
\end{example}

We define the problem Enum$\cdot\SR$ similarly to Enum$\cdot\IP$. A couple of results about the complexity of computing sufficient reasons have
been pointed out for the past few years. Obviously enough, when no assumption is made on the representation of $f$, computing a single sufficient reason for
an assignment $a$ is already {\sf NP}-hard (for pretty much the same reasons as for the prime implicant case, i.e., $f$ is valid iff for any $a$,
the unique sufficient reason for $a$ given $f$ is the empty term). Furthermore, the number of sufficient reasons for an assignment $a$ given $f$
can be exponential in the number of variables even when $f$ is represented in DT \cite{DBLP:journals/corr/abs-2108-05266}.
Contrary to abductive explanations, it is computationally easy to generate a single sufficient reason from $\SR(\Sigma, a)$ when $\Sigma$ is an OBDD circuit or a decision tree representing $f$.
A greedy algorithm can be used to this end: if $a$ satisfies $\Sigma$ (resp. $\neg \Sigma$), then start with the canonical term having $a$ as its unique satisfying assignment and remove literals from this term while ensuring that it still is an implicant of $\Sigma$ (resp. $\neg \Sigma$), until no more literal can be removed. In addition, when $\Sigma$ is in DT, we can generate in polynomial time a monotone CNF formula $\Psi$ such that $\IP(\Psi) =  \SR(\Sigma, a)$ (see \cite{DarwicheMarquis21} for details), and then take advantage of a quasi-polynomial time algorithm for enumerating the elements of $\IP(\Psi)$ \cite{DBLP:journals/dam/GurvichK99}.
Contrastingly, deciding whether a preset number of sufficient reasons for a given $a$ exists is intractable ({\sf NP}-hard), even when the Boolean function $f$ is monotone (see Theorem 3 in~\cite{MarqueGCIN21}). 
\medskip

In the following, we complete those results by providing evidence that Enum$\cdot SR$ from any language among dec-DNNF, OBDD, or DT is a difficult problem, despite the fact that those languages are quite convenient for many reasoning tasks \cite{DarwicheM02,Koricheetal13}.

Let us first give an inductive computation of $\SR(\Sigma,a)$ similar to that of $\IP(\Sigma)$.

\begin{proposition}\label{prop:decomposable_and_SR}
Let $f$ and $g$ be Boolean functions with $\var(f) \cap \var(g) = \emptyset$ and let $a$ be a truth assignment to a superset of $\var(f) \cup \var(g)$, then 
$\SR(f \land g, a) = \{t \land t' \mid t \in \SR(f,a), \, t' \in \SR(g,a) \}$.
\end{proposition}

\begin{proposition}\label{prop:decision_node_SR}
Let $f$ be a Boolean function, let $a$ be a truth assignment to a superset of $\var(f)$ and let $x \in \var(f)$. If $a$ satisfies the literal $\ell$ on variable $x$ then
\[
\begin{aligned}
\SR(f, a)
&= \{t \land \ell \mid t \in \SR(f|\ell,a), \, t \not\models f|\overline{\ell} \}  
\\
&\cup \, \SR(f|\overline{x} \land f|x, a).
\end{aligned}
\]
\end{proposition}

\noindent By Proposition~\ref{prop:number_of_IP}, $|\IP(f)| \geq \max(|\IP(f|\overline{x})|,|\IP(f|x)|)$. In a sense this means that using $\IP(f|\overline{x})$ and $\IP(f|x)$ to generate $\IP(f)$ is not a waste of resources since all these implicants are kept in some form through $\IP(f)$. This led to our output polynomial procedure to generate $\IP(f)$ for OBDD and more generally for dec-DNNF circuits. On the other hand, it is not guaranteed that $\SR(f,a)$ is larger than $\SR(f|x,a)$ and $\SR(f|\overline{x}, a)$ so there is no straightforward adaptation of this procedure from Enum$\cdot IP$ to Enum$\cdot SR$. 

\begin{example}
Let $\Sigma$ be the dec-DNNF circuit of Figure~\ref{fig:dec-DNNF}. Consider the dec-DNNF circuit $\Sigma_{v_1}$ rooted at node $v_1$, as spotted in Figure~\ref{fig:path}. The assignment $a$ to $\{b,e,p,s\}$ defined by $a(b) = a(e) = 1$ and $a(p) = a(s) = 0$ satisfies $\Sigma_{v_1}$. Recall that the set  $\IP(\Sigma_{v_1})$ has been constructed in Example 1 and observe that $\SR(\Sigma_{v_1},a) = \{\overline{p}\,\overline{s}\}$. Now the 0-child of $v_1$ is $v_2$ and looking at the set $\IP(\Sigma_{v_2})$ constructed in Example 1, we see that $\SR(\Sigma_{v_2},a) = \{\overline{p}\,\overline{s}, b\,\overline{p}\}$. Since $\Sigma_{v_2} = \Sigma_{v_1}|\overline{e}$, we have that $|\SR(\Sigma_{v_1}, a)| < \max(|\SR(\Sigma_{v_1}|\overline{e}, a)|,|\SR(\Sigma_{v_1}|e, a)|)$.
\end{example}
%
Actually, we give evidence that enumerating sufficient reasons  from dec-DNNF, and even from OBDD or DT, is not in {\sf OutputP} by reducing to it the problem of enumerating the minimal transversals of a hypergraph, a well-known problem whose membership to {\sf OutputP} is a long-standing question. Formally: 

\begin{proposition}
If Enum$\cdot\SR$ from OBDD is in {\sf OutputP} or Enum$\cdot\SR$ from DT is in {\sf OutputP}, then enumerating the minimal transversals of a hypergraph is in {\sf OutputP}. 
\end{proposition}

\section{Conclusion}

Most applications of prime implicants for Boolean function analysis use only a fraction of the many prime implicants a Boolean function may have. 
Especially, in the context of logic-based abduction, subset-minimal assumptions to be added to the available background
knowledge in order to be able to derive some given manifestations are looked for; in the propositional case, they correspond to specific prime implicants.
Furthermore, in an eXplainable AI perspective, specific prime implicants known as sufficient reasons are used to explain the predictions of 
machine learning algorithms. 

In our work, we have studied the enumeration of general and specific prime implicants of Boolean functions represented as dec-DNNF circuits. It was known that these circuits enable efficient reasoning on Boolean functions. Our results show that when it comes to prime implicants enumeration, dec-DNNF circuits have benefits as well as limitations. Our take-home message is that, while dec-DNNF circuits enable enumerating general prime implicants in incremental polynomial time, there are strong pieces of evidence against the existence of any output-polynomial time procedure for enumerating specific prime implicants from dec-DNNF circuits. More precisely, if a procedure for enumerating subset-minimal abductive explanations were to exist, then ${\sf P } = {\sf NP}$ would follow. Similarly, if there were an output-polynomial time algorithm for enumerating sufficient reasons from dec-DNNF circuits, then the enumeration of the minimal transversals of a hypergraph would be in {\sf OutputP}. Though this is considered unlikely in enumeration complexity, we think that proving a stronger statement would be a valuable contribution. We let this task open for future research.

\section*{Acknowledgments} 
Many thanks to the anonymous reviewers for their comments and insights. 
This work has benefited from the supports of the PING/ACK project (ANR-18-CE40-0011) and of 
the AI Chair EXPE\textcolor{orange}{KC}TATION (ANR-19-CHIA-0005-01) of the French National Research Agency.
It was also partially supported by TAILOR, a project funded by EU Horizon 2020 research
and innovation programme under GA No 952215.


\bibliographystyle{named}
\bibliography{enumIP.bib}

\section*{Appendix: Proofs}

\setcounter{proposition}{1}
\setcounter{lemma}{0}

\begin{proposition}
Enum$\cdot\IP$ from dec-DNNF is in {\sf EnumP}.
\end{proposition}

\begin{proof}
Direct from the fact that dec-DNNF is a sublanguage of \emph{deterministic} DNNF (d-DNNF) and
d-DNNF supports polynomial time implicant check~\cite{DarwicheM02}.
\end{proof}


\begin{proposition}
Let $f$ and $g$ be Boolean functions, then $\IP(f \land g) = \max(\{t \land t' \mid t \in \IP(f), \, t'\in \IP(g) \}, \models).$ Furthermore if $\var(f) \cap \var(g) = \emptyset$, then $\IP(f \land g) = \{t \land t' \mid t \in \IP(f), \, t'\in \IP(g) \}.$
\begin{proof}
A proof of $\IP(f \land g) = \max(\{t \land t' \mid t \in \IP(f), \, t'\in \IP(g) \}, \models)$ can be found e.g., in~\cite{Marquis93}. 

In the case where $\var(f) \cap \var(g) = \emptyset$, the terms in $\IP(f)$ contain only variables from $\var(f)$ and the terms in $\IP(g)$ contain only variables from $\var(g)$. We denote $lit(f) = \{x,\overline{x} \mid x \in \var(f)\}$. Let $t_f,t'_f \in \IP(f)$ and $t_g,t'_g \in \IP(g)$ such that $t_f \land t_g \models t'_f \land t'_g$. Looking at terms as sets of literals this means that $t'_f \cup t'_g \subseteq t_f \cup t_g$. But then $t'_f \subseteq t_f$ since  $(t_f \cup t_g) \cap lit(f) = t_f$ and $(t'_f \cup t'_g) \cap lit(f) = t'_f$. This means that $t_f \models t'_f$ and therefore $t_f = t'_f$. A similar argument gives that $t_g = t'_g$. This shows that when $\var(f) \cap \var(g) = \emptyset$, $\IP(f \land g) = \max(\{t \land t' \mid t \in \IP(f), \, t'\in \IP(g) \}, \models) = \{t \land t' \mid t \in \IP(f), \, t'\in \IP(g) \}$.
\end{proof}
\end{proposition}

\begin{proposition}
Let $f$ a Boolean function, let $x$ be a variable, and let $\ell \in \{x,\overline{x}\}$. Consider $t \in \IP(f|\ell)$. If $t \models f|\overline{\ell}$ then $t \in \IP(f)$, otherwise $t \land \ell \in \IP(f)$.
\begin{proof}
Suppose $t \models f|\overline{\ell}$. Then $t$ is an implicant of $f$ since $t \models (f|\overline{x} \land f|x) \models ((\overline{x} \land f|\overline{x}) \lor (x \land f|x)) \equiv f$. To prove that it is prime let $t'$ be a strict subterm of $t$ and assume $t' \models f$. We have $x \not\in \var(t)$ since $x \not\in \var(f|\ell)$, so $t'|\ell = t'$. If $t' \models f$ then $t' = t'|\ell \models f|\ell$ and $t$ is not a prime implicant of $f|\ell$, a contradiction.

Suppose $t \not\models f|\overline{\ell}$. Then $t \land \ell$ is an implicant of $f$ since $t\land\ell \models (\ell\land f|\ell) \models ((\overline{x} \land f|\overline{x}) \lor (x \land f|x)) \equiv f$. To prove that it is prime, let $t'$ be a strict subterm of $t \land \ell$ and assume $t' \models f$. If $\ell \not\in t'$ then $t' = t'|\ell \models f|\ell$ and $t$ is not a prime implicant of $f|\ell$, a contradiction. If however $\ell \in t'$, then write $t' = t'' \land \ell$ and observe that $t'' = t'|\ell \models f|\ell$, so $t$ is not a prime implicant of $f|\ell$, another contradiction.
\end{proof}
\end{proposition}

\begin{proposition}
Let $f$ be a Boolean function and let $x$ be a variable.
\[
\begin{aligned}
\IP(f)
= \,     &\{t \land \overline{x} \mid t \in \IP(f|\overline{x}), \, t \not\models f|x \} \\
 \cup \, &\{t \land x \mid t \in \IP(f|x), \, t \not\models f|\overline{x} \} \\
 \cup \, &\IP(f|\overline{x} \land f|x)
\end{aligned}
\]
\begin{proof}

We derive $\{t \land \overline{x} \mid t\in \IP(f|\overline{x}), t \not\models f|x\} \cup \{t \land x \mid t\in \IP(f|x), t \not\models f|\overline{x}\} \subseteq \IP(f)$ from Proposition~\ref{prop:propagate_IP}. 

Now we show that $\IP(f|\overline{x} \land f|x) \subseteq \IP(f)$. Let $t \in \IP(f|\overline{x} \land f|x)$, then $t \models f|\overline{x}$ and $t \models f|x$. Since $x \not\in \var(f|\overline{x}) \cup \var(f|x)$ we have that $x \not\in \var(t)$. First we prove that $t$ is an implicant of $f$. If we had $t \not\models f$ then $t|\overline{x} \not\models f|\overline{x}$ or $t|x \not\models f|x$ would hold, but $t|\overline{x} = t|x = t$ so $t \models f$. Now to prove that it is a prime implicant. Let $\ell \in t$ and let $t'$ be the term $t$ deprived from $\ell$. If $t' \models f$ were to hold then so would $t'|\overline{x} \models f|\overline{x}$ and $t'|\overline{x} \models f|x$. But since $t'|\overline{x} = t'|x = t'$, this would mean that $t' \models f|\overline{x} \land f|x$ and therefore $t$ would not be a prime implicant of $f|\overline{x} \land f|x$, a contradiction. This shows that $\IP(f|\overline{x} \land f|x) \subseteq \IP(f)$.

We have established that 
\[
\begin{aligned}
\IP(f)
\supseteq \,     &\{t \land \overline{x} \mid t \in \IP(f|\overline{x}), \, t \not\models f|x \} \\
 \cup \, &\{t \land x \mid t \in \IP(f|x), \, t \not\models f|\overline{x} \} \\
 \cup \, &\IP(f|\overline{x} \land f|x)
\end{aligned}
\]
and now we show the reverse inclusion. Let $t \in \IP(f)$. Assume $t = t_0 \land \overline{x}$. $t \models f$ implies that $t|\overline{x} \models f|\overline{x}$, or in other words, that $t_0 \models f|\overline{x}$. If $t_0$ was not a prime implicant of $ f|\overline{x}$, that is, if there was $t'_0 \subset t_0$ such that $t'_0 \models  f|\overline{x}$, then we would also have $t'_0 \land \overline{x} \models f$ and therefore $t$ would not be a prime implicant of $f$. So $t_0 \in \IP( f|\overline{x})$. Now if $t_0 \models f|x$ then we would have $t_0 \models f|\overline{x} \land f|x \models f$, so $t$ would not be a prime implicant of $f$, a contradiction. This shows that if $\overline{x}$ is in $t$, then $t \in \{t_0 \land \overline{x}\mid t_0 \in \IP(f|\overline{x}), t_0 \not\models f|x\}$. 
A symmetrical proof gives that if $x$ is in $t$, then $t \in \{t_1 \land x\mid t_1 \in \IP(f|x), t_1 \not\models f|\overline{x}\}$.

Finally if neither $x$ nor $\overline{x}$ is in $t$, then $t = t|\overline{x} \models f|\overline{x}$ and $t = t|x \models f|x$, and therefore $t \models f|\overline{x} \land f|x$. Now if $t$ was not a prime implicant of $f|\overline{x} \land f|x$ then there would be some $t' \subset t$ in $\IP(f|\overline{x} \land f|x)$. Since $\IP(f|\overline{x} \land f|x) \subseteq \IP(f)$, this would mean that $t$ is not a prime implicant of $f$, a contradiction. So $t \in \IP(f|\overline{x} \land f|x)$.

We have established that 
\[
\begin{aligned}
\IP(f)
\subseteq \,     &\{t \land \overline{x} \mid t \in \IP(f|\overline{x}), \, t \not\models f|x \} \\
 \cup \, &\{t \land x \mid t \in \IP(f|x), \, t \not\models f|\overline{x} \} \\
 \cup \, &\IP(f|\overline{x} \land f|x)
\end{aligned}
\]
thus finishing the proof.
\end{proof}
\end{proposition}

\begin{proposition}
Let $f$ a Boolean function and let $x$ be a variable, then 
$|\IP(f)| \geq \max(|\IP(f|\overline{x})|, |\IP(f|x)|)$.
\begin{proof}
Let $\ell \in \{\overline{x}, x\}$. It is shown in~\cite{DarwicheM02} that $\IP(f|\ell) = \max(\{t|\ell \mid t \in \IP(f)\},\models)$. So $|\IP(f|\ell)| = |\max(\{t|\ell \mid t \in \IP(f)\},\models)| \leq |\{t|\ell \mid t \in \IP(f)\}| = |\IP(f)|$.
\end{proof}
\end{proposition}

\begin{proposition}
Enum$\cdot\IP$ from dec-DNNF is in {\sf OutputP}.
\begin{proof}
Given a dec-DNNF circuit $\Sigma$, we construct $\IP(\Sigma)$ by visiting every node $v$ of $\Sigma$ in a bottom-up order while computing $\IP(\Sigma_v)$. We start from the leaves. If $v$ is labelled by a literal $\ell$ then $\IP(\Sigma_v) = \{\ell\}$, if it is labelled by $0$ then $\IP(\Sigma_v) = \emptyset$, and if it is labelled by $1$ then $\IP(\Sigma_v) = \{t_\emptyset\}$ where $t_\emptyset$ is the term containing no literal.

Now let $v$ be an internal node of $\Sigma$ and let $u$ and $w$ be its children. Since we visit the nodes in depth-first order, $\IP(\Sigma_u)$ and $\IP(\Sigma_w)$ have already been computed. If $v$ is a decomposable $\land$-node then using Proposition~\ref{prop:decomposable_and_IP} we compute $\IP(\Sigma_v) = \{t_u \land t_w \mid t_u \in \IP(\Sigma_u), t_w \in \IP(\Sigma_w)\}$ in time $O(|\IP(\Sigma_u)|\times|\IP(\Sigma_w)|) = O(|\IP(\Sigma)|^2)$. Observe that $|\IP(\Sigma_v)| \geq \max(|\IP(\Sigma_u)|,|\IP(\Sigma_w)|)$.

If $v$ is a decision node for the variable $x$ whose 0- and 1-children are $u$ and $w$, respectively, then we compute $\IP(\Sigma_v)$ from $\IP(\Sigma_{u})$ and $\IP(\Sigma_{w})$ using Proposition~\ref{prop:decision_node_IP}. Since dec-DNNFs support linear-time implicant check, $\{t \land \overline{x} \mid t \in \IP(\Sigma_u), t \not\models \Sigma_w\}$ and $\{t \land x \mid t \in \IP(\Sigma_w), t \not\models \Sigma_u\}$ are build in time polynomial in $|\Sigma|+|\IP(\Sigma_u)|+|\IP(\Sigma_w)|$. As for $\IP(\Sigma_u \land \Sigma_w)$, we build it in time polynomial in $|\IP(\Sigma_u)|+|\IP(\Sigma_w)|$ using Proposition~\ref{prop:decomposable_and_IP}. Observe that, by Proposition~\ref{prop:number_of_IP}, there is again $|\IP(\Sigma_v)| \geq \max(|\IP(\Sigma_u)|,|\IP(\Sigma_w)|)$.

When we reach the root node $r$ we compute $\IP(\Sigma_{r}) = \IP(\Sigma)$. Since for every node $v$ with children $u$ and $w$ we have that $|\IP(\Sigma_v)| \geq \max(|\IP(\Sigma_u)|,|\IP(\Sigma_w)|)$, we also have that $|\IP(\Sigma_v)| \leq |\IP(\Sigma)|$ for every $v$. We build $\IP(\Sigma_v)$ only once and the time spent on each node $v$ to build $\IP(\Sigma_v)$ given $\IP(\Sigma_u)$ and $\IP(\Sigma_w)$ is polynomial in $|\Sigma| + |\IP(\Sigma)|$. Summing over all nodes we get that the time needed to build $\IP(\Sigma)$ is also polynomial in $|\Sigma|+|\IP(\Sigma)|$.
\end{proof}
\end{proposition}

\begin{proposition}
Let $\Sigma$ be a dec-DNNF circuit and let $S \subseteq \IP(\Sigma)$. If the root of $\Sigma$ is an $\land$-node, let $u$ and $w$ be its children and let $S_u := \{t_{\var(\Sigma_u)} \mid t \in S\}$ and $S_w := \{t_{\var(\Sigma_w)} \mid t \in S\}$. Then $S_u \subseteq \IP(\Sigma_u)$ and $S_w \subseteq \IP(\Sigma_w)$ hold, and
\[
\begin{aligned}
S = \IP(\Sigma) &\text{ iff } S_u = \IP(\Sigma_u) \text{ and } S_w = \IP(\Sigma_w)
\\&\text{ and } S = \{t_u \land t_v \mid t_u \in S_u, t_v \in S_v\}.
\end{aligned}
\]
\begin{proof}
If $S = \IP(\Sigma)$ then by Proposition~\ref{prop:decomposable_and_IP} $S = \{t_u \land t_v \mid t_u \in \IP(\Sigma_u), t_v \in \IP(\Sigma_v)\}$ so $\IP(\Sigma_u) = \{t_{\var(\Sigma_u)} \mid t \in S\} = S_u$ and $\IP(\Sigma_v) = \{t_{\var(\Sigma_v)} \mid t \in S\} = S_v$ and thus $S = \{t_u \land t_v \mid t_u \in S_u, t_v \in S_v\}$.

If $S \neq \IP(\Sigma)$, let $t^* \in \IP(\Sigma)\setminus S$ and let $t^*_u = t^*_{\var(\Sigma_u)}$ and $t^*_v = t^*_{\var(\Sigma_v)}$. By Proposition~\ref{prop:decomposable_and_IP}, $t^*_u$ (resp. $t^*_v$) is in $\IP(\Sigma_u)$ (resp. $\IP(\Sigma_v)$), so either $t^*_u \not\in S_u$ or $t^*_v \not\in S_v$and we are done, or $t^*_u \in S_u$ and $t^*_v \in S_v$ in which case $ \{t_u \land t_v \mid t_u \in S_u, t_v \in S_v\} \neq S$ since $t^*_u \land t^*_v$ is in $\{t_u \land t_v \mid t_u \in S_u, t_v \in S_v\}$ but not in $S$.
\end{proof}
\end{proposition}

\begin{proposition}
Let $\Sigma$ be a dec-DNNF circuit whose root is a decision node labelled by $x$ and whose 0- and 1-child are $u$ and $w$. Given $S \subseteq \IP(\Sigma)$, let $S_u = \{t \mid t \land \overline{x} \in S\}  \cup (S \cap \IP(\Sigma_{u}))$, $S_w = \{t \mid t \land x \in S\} \cup (S \cap \IP(\Sigma_{w}))$, $S' = \{t \mid t \in S, var(x) \not\in var(t)\}$. Then $S_u \subseteq \IP(\Sigma_{u})$ and $S_w \subseteq \IP(\Sigma_{w})$ hold, and 
\[\begin{aligned}
S = \IP(\Sigma) \text{ iff } & S_u = \IP(\Sigma_{u}) \\ \text{ and } & S_w = \IP(\Sigma_{w}) \\ \text{ and } & S' = \max(\{t_u \land t_w \mid t_u \in S_u, \, t_w\in S_w \}, \models).
\end{aligned}
\]
\begin{proof}
For convenience we denote $S^* = \max(\{t_u \land t_w \mid t_u \in S_u, \, t_w \in S_w \}, \models)$.

First we prove that $S_w \subseteq \IP(\Sigma_w)$ (the proof that $S_u \subseteq \IP(\Sigma_u)$ is analogous). Clearly $S \cap \IP(\Sigma_w) \subseteq \IP(\Sigma_w)$ so we just need to show that $\{t \mid t \land x \in S\} \subseteq \IP(\Sigma_w)$. Let $t\land x$ be in $S$, then $t \land x \models \Sigma$. $t$ is an implicant of $\Sigma_w$ since $t\equiv (t \land x)|x \models \Sigma|x \equiv \Sigma_w$. Now if there exists $t' \neq t$ such that $t \models t' \models \Sigma_w$ then $t \land x \models t'\land x \models \Sigma$ holds, and therefore $t \land x$ is not a prime implicant of $\Sigma$, a contradiction. So $\{t \mid t \land x \in S\} \subseteq \IP(\Sigma_w)$.

Second we prove that $S_w  \neq \IP(\Sigma_w)$ implies $S \neq \IP(\Sigma)$ (the proof is similar for $S_u \neq \IP(\Sigma_u)$). Assume there exists $t \in \IP(\Sigma_w) \setminus S_w$. If $t \models \Sigma_u$ then $t$ is in $\IP(\Sigma)$ by Proposition~\ref{prop:propagate_IP}. But $t$ cannot be in $S$ for otherwise it would be in $S \cap \IP(\Sigma_w) \subseteq S_w$. This shows that $S \neq \IP(\Sigma)$ in this case. If however $t \not\models \Sigma_u$ then $t \land x$ is in $\IP(\Sigma)$ by Proposition~\ref{prop:propagate_IP}. But $t \land x$ cannot be in $S$ for otherwise $t$ would be in $\{\tau \mid \tau \land x \in S\} \subseteq S_w$. So again $S \neq \IP(\Sigma)$. 

Now we prove that $(S' \neq S^*) \Rightarrow (S \neq \IP(\Sigma))$. We may assume that $S_u = \IP(\Sigma_u)$ and $S_w = \IP(\Sigma_w)$, otherwise $S \neq \IP(\Sigma)$ holds regardless of $S' = S^*$. Since $\Sigma_u \equiv \Sigma|\overline{x}$ and $\Sigma_w = \Sigma|x$ we have that $S^* = \max(\{t_u \land t_w \mid t_u \in \IP(\Sigma|\overline{x}), \, t_w \in \IP(\Sigma|x) \}, \models) = \IP(\Sigma|\overline{x} \land \Sigma|x)$ by Proposition~\ref{prop:decomposable_and_IP}. Now $S = S' \cup \{t \mid t\in S, \overline{x} \in t\} \cup \{t \mid t\in S, x\in t\}$ so, by Proposition~\ref{prop:decision_node_IP}, if $S = \IP(\Sigma)$ then $S'$ corresponds to the set $\IP(\Sigma|\overline{x} \land \Sigma|x)$. So 
$$
(S' \neq S^*) \Rightarrow (S' \neq \IP(\Sigma|\overline{x} \land \Sigma|x)) \Rightarrow  (S \neq \IP(\Sigma))
$$

Now for the other direction, assume there exists $t \in \IP(\Sigma) \setminus S$. First suppose that $t = t' \land x$. On the one hand $t'$ is in $\IP(\Sigma|x) = \IP(\Sigma_w)$. On the other hand $t'$ is clearly not in $\{\tau \mid \tau \land x \in S\}$, and since it is not an implicant of $\Sigma$, it is not in $S \cap \IP(\Sigma_w)$ either. This means that $t' \in \IP(\Sigma_w) \setminus S_w$ and therefore $S_w \neq \IP(\Sigma_w)$. In the case where $t = t' \land \overline{x}$, a similar proof gives that $S_u \neq \IP(\Sigma_u)$. It remains to consider the situation where neither $x$ nor $\overline{x}$ is in $t$. By Proposition~\ref{prop:decision_node_IP}, $t$ is contained in $\IP(\Sigma|\overline{x} \land \Sigma|x)$. As before, we can assume that $S_u = \IP(\Sigma_u)$ and $S_w = \IP(\Sigma_w)$. We have already explained that this assumption yields $S^* = \IP(\Sigma|\overline{x} \land \Sigma|x)$. Since $t$ is not in $S$ and $\overline{x} \not\in t$ and $x\not\in t$, we have that $t \not\in S'$. So $t \in S^* \setminus S'$, and therefore $S \neq S^*$.  
\end{proof}
\end{proposition}

\begin{algorithm}[t]
\SetAlgoLined
\nonl\textbf{Promise}: $\Sigma$ is satisfiable\\
Find a satisfying assignment $a$ of $\Sigma$
\\
Let $t$ be the corresponding term: $t = \bigwedge_{a(x) = 1} x \land \bigwedge_{a(x) = 0} \overline{x} $
\\
\While{\textup{there is} $\ell \in t$ \textup{such that} $t - \ell \models \Sigma$}{
Remove $\ell$ from $t$
}
return $t$
\caption{\texttt{GenerateIP}$(\Sigma)$}
\end{algorithm}

\begin{proposition} Given a reduced dec-DNNF circuit $\Sigma$ and $S \subseteq \IP(\Sigma)$, 
\texttt{\textup{MissingIP}}$(\Sigma,S,\emptyset)$ runs in time $O(poly(|S|+|\Sigma|))$, and it returns \textit{false} if and only if $S = \IP(\Sigma)$.
\begin{proof}
\textbf{Soundness:} We prove soundness by induction on the depth of $\Sigma$ using Propositions~\ref{prop:decomposable_and_IP_check} and~\ref{prop:decision_node_IP_check}. 
 
If $\Sigma$ has depth 1 then it is a single node $v$ labelled by 0, 1 or a literal $\ell$. The promise states that $S \subseteq \IP(\Sigma)$. If $v$ is labelled by 0, then $S$ must be $\emptyset$ and the algorithm returns \textit{false} at line 4. If $v$ is is labelled by 1 then either $S = \{t_\emptyset\} = \IP(1)$ and the algoritm returns \textit{false} at line 24, or $S = \emptyset$ and the algorithm returns $(1,(v))$ at line 5. Finally if $v$ is labelled by $\ell$, either $S = \{\ell\} = \IP(\ell)$ and the algorithm returns \textit{false} at line 24, or $S = \emptyset$ and the algorithm returns $(\ell,(v))$ at line 5. In all cases the algorithm returns \textit{false} if and only if $S = \IP(\Sigma)$, and it sets $\lambda(v)$ to $|\IP(\Sigma_v)|$ before returning \textit{false}.

Now if $\Sigma$ has depth more than 1, its root node $v$ is either a decomposable $\land$-node or a decision node. Since $\Sigma$ is reduced, it cannot be unsatisfiable, so if $S = \emptyset$ the algorithm returns $(t,(v))$ with $t \in \IP(\Sigma)$ at line 5. From now on we suppose that $S \neq \emptyset$. If $v$ is a decomposable $\land$-node with children $u$ and $w$. By Proposition~\ref{prop:decomposable_and_IP_check}, since we are promised that $S \subseteq \IP(\Sigma)$, we have that $\IP(\Sigma) = S$ if and only if  $\IP(\Sigma_u) = S_u$ and $\IP(\Sigma_w) = S_w$ and we can construct $S$ from $S_u$ and $S_w$ as shown in Proposition~\ref{prop:decomposable_and_IP_check} ($S_u$ and $S_w$ defined as in Proposition~\ref{prop:decomposable_and_IP_check}). By induction $\IP(\Sigma_u) \neq S_u$ or $\IP(\Sigma_w) \neq S_w$ if and only if the output of $\texttt{MissingIP}(\Sigma_u,S_u,\ast)$ or $\texttt{MissingIP}(\Sigma_w,S_w,\ast)$ is distinct from \textit{false}. So if $\IP(\Sigma_u) \neq S_u$ or $\IP(\Sigma_w) \neq S_w$, a return statement occurs line 9 or 12. Otherwise, it possible that $\IP(\Sigma_u) = S_u$ or $\IP(\Sigma_w) = S_w$ but that $S$ can not be constructed from $S_u$ and $S_w$, then the return statement of line 14 is triggerd. So if $S \neq \IP(\Sigma)$, then lines 8-14 return something that is not \textit{false}. And if $S = \IP(\Sigma)$, then no return call is triggered lines 8-14 and the algorithm returns \textit{false} at line 24 after setting $\lambda(v)$ to $|S|=|\IP(\Sigma)|=|\IP(\Sigma_v)|$.

If $v$ is a decision node for variable $x$ with 0-child $u$ and 1-child $w$. By Proposition~\ref{prop:decision_node_IP_check}, since we are promised that $S \subseteq \IP(\Sigma)$, we have that $\IP(\Sigma) = S$ if and only if  $\IP(\Sigma_u) = S_u$ and  $\IP(\Sigma_w) = S_w$ and $S' = S^*$ with $S_u$, $S_w$ and $S'$ defined as in Proposition~\ref{prop:decision_node_IP_check} and $S^*$ defined line 21.  By induction $\IP(\Sigma_u) = S_u$ and  $\IP(\Sigma_w) = S_w$ if and only if the output of $\texttt{MissingIP}(\Sigma_u,S_u,\ast)$ or $\texttt{MissingIP}(\Sigma_w,S_w,\ast)$ is distinct from \textit{false}. So if $S \neq \IP(\Sigma)$, then lines 16-22 return something that is not \textit{false}. And if $S = \IP(\Sigma)$, then no return call is triggered lines 16-22 and the algorithm returns \textit{false} at line 24 after setting $\lambda(v)$ to $|S|=|\IP(\Sigma)|=|\IP(\Sigma_v)|$.

\medskip 
\noindent \textbf{Running time:} Consider the time spent in $\texttt{MissingIP}(\Sigma,S,P)$ before a return statement or a recursive call is triggered. The procedure may end at line 2 or 4 in $O(1)$ time. It can also end line 5, in which case it has to compute a prime implicant of $\Sigma$ using \texttt{GenerateIP}, which runs in time $O(poly(|\Sigma|))$. Now if the algorithm has not returned lines 2, 4 or 5, most of the running time is spent building sets of terms from $S$ lines 8, 13, 16 and 21. Building $S_u$ and $S_w$ line 8 only requires projecting the terms in $S$ onto $\var(\Sigma_u)$ and $\var(\Sigma_w)$, which takes time $O(|S|)$. Constructing the set $S^*$ at line 13 takes $O(|S_u|\times|S_w|) = O(|S|^2)$ time. At line 16, $S'$ can clearly be obtained in time $O(|S|)$ and $S_u$ and $S_w$ are obtained in time $O(poly(|S| + |\Sigma|))$ thanks to polynomial-time prime implicant check on dec-DNNF circuits. Finally the set $S^*$ at line 21 is constructed in $O(|S_u|\times|S_w|) = O(|S|^2)$ and compared to $S'$ in time $O(poly(|S|)$. So before a return statement or a recursive call is triggered, the algorithm spends $O(poly(|S| + |\Sigma|)$ time. One can observe that $|S_u|$, $|S_w|$ are fewer than $|S|$, so for every node $v$ in $\Sigma$, a call $\texttt{MissingIP}(\Sigma_v,S',*)$ takes $O(poly(|S| + |\Sigma|))$ time before triggering a return statement or a recursive call. Thanks to memoization -- implemented via $\lambda$ -- the $O(poly(|S| + |\Sigma|))$ time procedure is done only once per node. So the total running time of the algorithm is also in  $O(poly(|S| + |\Sigma|))$.
\end{proof}
\end{proposition}

\begin{proposition}
Let $\Sigma$ be a reduced dec-DNNF circuit and let $S \subseteq \IP(\Sigma)$. \texttt{\textup{AnotherIP}}$(\Sigma,S)$ runs in time $O(poly(|S|+|\Sigma|))$. It  returns $\textit{false}$ if $S = \IP(\Sigma)$, otherwise it returns a prime implicant of $\Sigma$ that does not belong to $S$.
\begin{proof}
\textbf{Soundness.} First \texttt{\textup{AnotherIP}}$(\Sigma,S)$ calls \texttt{\textup{MissingIP}}$(\Sigma,S,\emptyset)$. Soundness of \texttt{MissingIP} has been established in Proposition~\ref{prop:missingIP_is_sound} so if $S = \IP(\Sigma)$ then \texttt{\textup{MissingIP}}$(\Sigma,S,\emptyset)$ returns \textit{false} and so does \texttt{\textup{AnotherIP}}$(\Sigma,S)$.

Now let us assume that \texttt{\textup{MissingIP}}$(\Sigma,S,\emptyset)$ has not returned \textit{false} but the pair $(t,P)$ with $P = (v_0,\dots,v_m)$ a path from $v_0$ (the root of $\Sigma$) to $v_m$ and $t$ a term. Use the notation $P_i = (v_0,\dots,v_{i-1})$ for all $1 \leq i \leq m$. Then calling \texttt{\textup{MissingIP}}$(\Sigma,S,\emptyset)$ has triggered a sequence of recursive calls \texttt{\textup{MissingIP}}$(\Sigma_{v_1},S_1,P_1)$, \texttt{\textup{MissingIP}}$(\Sigma_{v_2},S_2,P_2)$,$\dots$,\texttt{\textup{MissingIP}}$(\Sigma_{v_m},S_m,P_m)$. A contradiction has been found during the last step: \texttt{\textup{MissingIP}}$(\Sigma_{v_m},S_m,P_m)$ ended line 5 for a contradiction of type (c1), or line 20 for a contradiction of type (c2), and returned $(t,P)$ with $t$ some term that we claim is in  $\IP(\Sigma_{v_m})\setminus S_m$.

\begin{claim}
$t \in \IP(\Sigma_{v_m})\setminus S_m$. 
\begin{proof}
This is clear if \texttt{\textup{MissingIP}}$(\Sigma_{v_m},S_m,P_m)$ ends line 5. Now if it ends line 20, then $v_m$ is a decision node for $x$ with 0-child $u$ and 1-child $w$. The sets $S_u$, $S_w$, $S'$ and $S^*$ have been generated and that it has been shown that $S_u = \IP(\Sigma_u)$ and $S_w = \IP(\Sigma_w)$ (otherwise a return statement line 16 or 18 would have been triggered). So $S^* = \IP(\Sigma_u \land \Sigma_w) = \IP(\Sigma_{v_m}|\overline{x}\land \Sigma_{v_m}|x)$ by Proposition~\ref{prop:decomposable_and_IP}. We have $t \in S^* \setminus S'$ so it is clear that $x \not\in \var(t)$. Furthermore $S'$ contains all terms from $S_m$ in which neither $x$ nor $\overline{x}$ appears, so $t \in S^* \setminus S'$ really means that $t \in S^* \setminus S_m = \IP(\Sigma_{v_m}|\overline{x}\land \Sigma_{v_m}|x) \setminus S_m \subseteq \IP(\Sigma_{v_m}) \setminus S_m$. 
\end{proof}
\end{claim}

Now \texttt{\textup{AnotherIP}}$(\Sigma,S)$ returns the result \texttt{\textup{Propagate}}$(\Sigma,t,P)$. To prove that the output is a term in $\IP(\Sigma) \setminus S$, it is sufficient to show that, for every $1 \leq i \leq m$, if $t_i \in \IP(\Sigma_{v_i})\setminus S_i$ then \texttt{\textup{Propagate}}$(\Sigma,t_i,P_i)$ calls \texttt{\textup{Propagate}}$(\Sigma,t_{i-1},P_{i-1})$ with $t_{i-1} \in  \IP(\Sigma_{v_{i-1}})\setminus S_{i-1}$. The rest is an easy induction (with $S_0 = S$ and $\Sigma_{v_0} = \Sigma$).

\begin{claim}
Let $t_i \in \IP(\Sigma_{v_i}) \setminus S_i$ with $i \geq 1$ then \texttt{\textup{Propagate}}$(\Sigma,t_i,P_i)$ does a recursive call \texttt{\textup{Propagate}}$(\Sigma,t_{i-1},P_{i-1})$ with $t_{i-1} \in  \IP(\Sigma_{v_{i-1}})\setminus S_{i-1}$.
\begin{proof}
\texttt{\textup{Propagate}}$(\Sigma,t_i,P_i)$ calls \texttt{\textup{Propagate}}$(\Sigma,t_i \land t',P_{i-1})$. Let $t_{i-1} = t_i \land t'$. We need to show that it is in $\IP(\Sigma_{v_{i-1}})\setminus S_{i-1}$. First assume that $v_{i-1}$ is a decomposable $\land$-node with children $v_i$ and $w$, then $t'$ is obtained line 4 and clearly $t' \in \IP(\Sigma_w)$. By Proposition~\ref{prop:decomposable_and_IP}, $t_i \land t' \in \IP(\Sigma_{v_i} \land \Sigma_w) = \IP(\Sigma_{v_{i-1}})$. By construction $S_i = \{t_{\var(\Sigma_{v_i})} \mid t \in S_{i-1}\}$. If $t_i \land t'$ was in $S_{i-1}$ then its restriction $t_i$ to $\var(\Sigma_{v_i})$ would be $S_i$, a contradiction. So $t_i \land t' \not\in S_{i-1}$.

Now suppose $v_{i-1}$ is a decision node for $x$ with 0-child $u$ and 1-child $w$. Let $v_i$ be $u$ (the case $v_i = w$ is analogous). By construction $S_i = S_u$. $t'$ is obtained line 7 and, by Proposition~\ref{prop:propagate_IP}, $t_i \land t' \in \IP(\Sigma_{v_{i-1}})$.  To prove that $t_i \land t' \not\in S_{i-1}$, first assume that $t_i \models \Sigma_w$. Then $t'$ is the empty term $t_\emptyset$. So $t_i \land t' = t_i$ and $t_i \in \IP(\Sigma_{v_{i-1}})$.  If $t_i$ was in $S_{i-1}$ then we would have $t_i \in S_{i-1} \cap \IP(\Sigma_{v_i}) \subseteq S_i$, a contradiction. So when $t_i \models \Sigma_w$, we have $t_i \land t' \in  \IP(\Sigma_{v_{i-1}})\setminus S_{i-1}$. Now if $t_i \not\models \Sigma_w$, then $t_i \land t' = t_i\land \overline{x}$ and $t_i \land \overline{x}$ is not in $S_{i-1}$ for otherwise we would have $t_i \in \{\tau \mid \tau \land \overline{x} \in S_{i-1}\} \subseteq S_i$.  So again we have $t_i \land t' \in  \IP(\Sigma_{v_{i-1}})\setminus S_{i-1}$.
\end{proof}
\end{claim}

\medskip 
\noindent \textbf{Running time.} It has already been proved in Proposition~\ref{prop:missingIP_is_sound} that \texttt{Missing}($\Sigma$,$S$,$\emptyset$) runs in time $O(poly(|S| + |\Sigma|)$. As for \texttt{\textup{Propagate}}$(\Sigma,t,P)$, $|P|$ recursive calls are made and the cost between two consecutive recursive calls is either one call to \texttt{GenerateIP} line 3 or 4, or one implicant check line 7 or 9. An implicant test on a dec-DNNF takes linear time and \texttt{GenerateIP} makes at most $|\var(\Sigma)|$ such tests, so it runs in time $O(poly(|\Sigma|))$. Thus  \texttt{\textup{Propagate}}$(\Sigma,t,P)$ runs in time $O(|P|\times poly(|\Sigma|)) = O(poly(|\Sigma|))$.
\end{proof}
\end{proposition}

\begin{proposition}\label{lemma:IP_enum_for_dec_DNNF_is_in_IncP}
Enum$\cdot\IP$ from dec-DNNF is in {\sf IncP}.
\begin{proof}
Using Proposition \ref{prop:AnotherIP}, $k$ prime implicants of $\Sigma$ can be generated in time $O(poly(k+|\Sigma|))$ by simply calling \texttt{AnotherIP}$(\Sigma,S)$ $k$ times, each time adding to $S$ the new prime implicant that has been computed. This shows that Enum$\cdot\IP$ from dec-DNNF is in {\sf IncP}.
\end{proof}
\end{proposition}  

\begin{proposition}\label{prop:abductive}
Unless ${\sf P} = {\sf NP}$, there is no polynomial-time algorithm which, given an OBDD circuit or a decision tree computing a function $f$ over $X$ and a set $Y \subseteq X$, decides whether $f$ has an implicant $t$ with $\var(t) \subseteq Y$. 
\begin{proof}
Let $\phi$ be a CNF formula with $m$ clauses $c_1,\dots,c_m$. Create $m$ fresh variables $z_1,\dots,z_m$. Let $B_1,\dots,B_m$ be OBDD circuits respecting the same variable ordering and computing $c_1,\dots,c_m$, respectively. These OBDD circuits can be computed in polynomial time (and can even be chosen in DT). Define now the OBDD circuits $B^{(i)} = (\overline{z_i}\land B_i)  \lor (z_i\land B^{(i+1)})$ for $1\leq i \leq m$, with $B^{(m+1)} = 1$. $B^{(1)}$ is an OBDD circuit on $\{z_1,\dots,z_m\} \cup \var(\phi)$ built in polynomial time from $\phi$ and whose size is in $O(|\phi|)$. 
\begin{claim} An implicant $t$ of $B^{(1)}$ such that $\var(t) \subseteq \var(\phi)$ exists if and only if $\phi$ is satisfiable.
\begin{proof}
For the first direction assume the implicant exists. $t$ is an implicant of $B^{(1)} = (\overline{z_1}\land B_1)  \lor (z_1\land B^{(2)})$. Since $z_1 \not\in \var(t)$, we have $t \models B_1 \equiv c_1$ and $t \models B^{(2)}$. Following the same line of reasoning with $B^{(2)}$ instead of $B^{(1)}$ we also have that $t \models B_2 \equiv c_2$ and $t \models B^{(3)}$. And we repeat the argument until reaching, $t \models c_1$, $t \models c_2$, $\dots$, $t \models c_{m}$, $t \models B^{(m+1)} = 1$. So indeed $t \models \phi$ and then $\phi$ is satisfiable.

For the other direction assume $\phi$ is satisfiable. Then there exists an implicant $t$ of $\phi$ with $\var(t) \subseteq \var(\phi)$. Let $a$ be a truth assignment to $\var(\phi) \cup \{z_1,\dots,z_m\}$ that satisfies $t$. If $a(z_i) = 1$ for all $i \in \{1,\dots, m\}$, then $B^{(1)}|a \equiv B^{(2)}|a \equiv \dots \equiv B^{(m+1)}|a \equiv 1$. Otherwise let $j$ be the smallest integer such that $a(z_j) = 0$. Then $B^{(1)}|a \equiv B^{(2)}|a \equiv \dots \equiv B^{(j)}|a \equiv B_j|a \equiv c_j|a$. Since $t$ is an implicant of $\phi$, we have that $t \models c_j$, so $c_j|a \equiv 1$. Thus every assignment $a$ that satisfies $t$ also satisfies $B^{(1)}$, in other words $t$ is an implicant of $B^{(1)}$.
\end{proof}
\end{claim}
So if the algorithm from the proposition statement exists, we can run it on inputs $B^{(1)}$ and $Y = \var(\phi)$ to decide in polynomial time whether $\phi$ is satisfiable.

Finally note that if one had chosen to represent $B_1,\dots,B_m$ as decision trees from DT (which is also feasible in polynomial time), then $B^{(1)}$ would be an element of DT. So the statement also holds for DT.
\end{proof}
\end{proposition}

\begin{proposition}
Let $f$ and $g$ be Boolean functions with $\var(f) \cap \var(g) = \emptyset$ and let $a$ be a truth assignment to a superset of $\var(f) \cup \var(g)$, then 
$\SR(f \land g, a) = \{t \land t' \mid t \in \SR(f,a), \, t' \in \SR(g,a) \}$.
\begin{proof}
Comes from  Proposition~\ref{prop:decomposable_and_IP}:
\[
\begin{aligned}
&\SR(f \land g,a) 
= \{\tau \in \IP(f \land g) \mid a \textup{ satisfies } \tau\}
\\
&= \{t \land t' \mid t \in \IP(f), t' \in \IP(g), a \textup{ satisfies } t\land t'\}
\\
&= \{t \land t' \mid t \in \IP(f), t' \in \IP(g), a \textup{ satisfies both } t \textup{ and } t'\}
\\
&= \{t \land t' \mid t \in \SR(f,a), \, t' \in \SR(g,a) \}
\end{aligned}
\]
\end{proof}
\end{proposition}

\begin{proposition}
Let $f$ be a Boolean function, let $a$ be a truth assignment to a superset of $\var(f)$ and let $x \in \var(f)$. If $a$ satisfies the literal $\ell$ on variable $x$ then
\[
\begin{aligned}
\SR(f, a)
&= \{t \land \ell \mid t \in \SR(f|\ell,a), \, t \not\models f|\overline{\ell} \}  
\\
&\cup \, \SR(f|\overline{x} \land f|x, a).
\end{aligned}
\]
\begin{proof} Comes from  Proposition~\ref{prop:decision_node_IP}:
\[
\begin{aligned}
\SR(f,a) 
= &\{t \in \IP(f) \mid a \textup{ satisfies } t\}
\\
= &\{t \land \overline{\ell} \mid t \in \IP(f|\overline{\ell}), \, t \not\models f|\ell, a \textup{ satisfies } t\} 
\\ &\cup \{t \land \ell \mid t \in \IP(f|\ell), \, t \not\models f|\overline{\ell}, a \textup{ satisfies } t \} \\
& \cup \{t \in \IP(f|\overline{x} \land f|x) \mid , a \textup{ satisfies } t\}
\\
= &\{t \land \ell \mid t \in \SR(f|\ell,a), \, t \not\models f|\overline{\ell} \}  
\\
&\cup \SR(f|\overline{x} \land f|x, a).
\end{aligned}
\]
\end{proof}
\end{proposition}

\begin{proposition}
If Enum$\cdot\SR$ from OBDD is in {\sf OutputP} or Enum$\cdot\SR$ from DT is in {\sf OutputP}, then enumerating the minimal transversals of a hypergraph is in {\sf OutputP}. 
\begin{proof}
The proof leans on the proof of Theorem 2 in~\cite{KavvadiasPS93}. Let $\calH$ be an hypergraph. Vertices are identified by integers $1,\dots,n$ and associated to variables $x_1,\dots,x_n$. Let $tr(\calH)$ be the set of transversals of $\calH$ and let $tr_{\textup{min}}(\calH)$ be the set of minimal transversals of $\calH$. For each $S \subseteq \{1,\dots,n\}$ of vertices let $a_S$ be the assignment such that $a_S(x_i) = 0$ if and only if $i \in S$, and let $\gamma_S = \bigvee_{i \in S} \overline{x_i}$. Observe that $a_S$ satisfies $\gamma_{S'}$ if and only if $S \cap S' \neq \emptyset$. Let $f$ be the function whose satisfying assignments are exactly the $a_H$ for $H \in \calH$. Denote by $sat(f)$ the set of satisfying assignments of $f$. 

Now we have the following:
\[
\begin{aligned}
f \models \gamma_S 
&\Leftrightarrow \forall H \in \calH, a_H \textup{ satisfies } \gamma_S
\\
&\Leftrightarrow \forall H \in \calH, H \cap S \neq \emptyset
\\
&\Leftrightarrow S \textup{ is a transversal of } \calH
\end{aligned}
\]

This means that the set of implicates of $f$ containing only negative literals is $\{\gamma_T \mid T \in tr(\calH)\}$, and that the set of \textit{prime} implicates of $f$ containing only negative literals is $\{\gamma_T \mid T \in  tr_{\textup{min}}(\calH)\}$. Since the prime \textit{implicants} of $\neg f$ are exactly the negation of the prime \textit{implicates} of  $f$, we get that the set of prime implicants of $\neg f$ containing only positive literals is $\{\bigwedge_{i \in T} x_i \mid T \in  tr_{\textup{min}}(\calH)\}$. Observe that $a_\emptyset$ is the assignment that set all $x_i$ to 1 and that 
$$
\SR(\neg f, a_\emptyset) = \left\{\bigwedge\nolimits_{i \in T} x_i \mid T \in  tr_{\textup{min}}(\calH)\right\}.
$$
From $\calH$ we construct $sat(f)$ in polynomial time. Then from $sat(f)$ we construct in polynomial time an OBDD circuit $B$ equivalent to $f$. Then we obtain an OBDD $B'$ equivalent to $\neg f$ by switching the 0-sink and the 1-sink of $B$. Given the bijection between $\SR(B',a_\emptyset)$ and $tr_{\textup{min}}(\calH)$, any algorithm for enumerating sufficient reasons from OBDD can be run with inputs $B'$ and $a_\emptyset$ to enumerate the minimal transversals of $\calH$. So if Enum$\cdot\SR$ from OBDD is in {\sf OutputP} then enumerating the minimal transversals of a hypergraph is in {\sf OutputP}.

Finally, note that from $sat(f)$ one can construct a decision tree representing $f$ in polynomial time (instead of an OBDD circuit), and that negating such a decision tree boils down to turning 0-leaves into 1-leaves and vice-versa. So the statement also holds for Enum$\cdot\SR$ from DT.
\end{proof}
\end{proposition}
\end{document}